\documentclass[a4paper]{article}

\usepackage[utf8x]{inputenc}
\usepackage[T1]{fontenc}

\usepackage[a4paper,margin=1in]{geometry}

\usepackage[numbers]{natbib}
\usepackage[usenames,dvipsnames]{xcolor}

\definecolor{DarkRed}{rgb}{0.368,0.097,0.078}

\usepackage[colorlinks = true,
    linkcolor = blue,
    anchorcolor = blue,
    citecolor = blue,
    filecolor = blue,
    urlcolor = DarkRed,
    pagebackref]{hyperref}

\usepackage[T1]{fontenc}

\usepackage{amsthm} 
\usepackage{mathtools,thmtools}

\usepackage{enumitem}
\usepackage{bm}
\usepackage{xspace}
\usepackage{nicefrac}

\usepackage[capitalise]{cleveref}
\usepackage{dsfont}

\declaretheoremstyle[
	    spaceabove=\topsep, 
	    spacebelow=\topsep, 
	    headfont=\normalfont\bfseries,
	    bodyfont=\normalfont\itshape,
	    notefont=\normalfont\bfseries,
	    notebraces={(}{)},
	    postheadspace=0.5em, 
	    headpunct={},
	    postfoothook=\noindent\ignorespaces
    ]{theorem}
\declaretheorem[style=theorem,numberwithin=section]{theorem}

\declaretheoremstyle[
	    spaceabove=\topsep, 
	    spacebelow=\topsep, 
	    headfont=\normalfont\bfseries,
	    bodyfont=\normalfont,
	    notefont=\normalfont\bfseries,
	    notebraces={(}{)},
	    postheadspace=0.5em, 
	    headpunct={},
	    postfoothook=\noindent\ignorespaces
    ]{definition}

\declaretheoremstyle[
        spaceabove=\topsep, 
        spacebelow=\topsep, 
        headfont=\normalfont\bfseries,
        bodyfont=\normalfont,
        notefont=\normalfont\bfseries,
        notebraces={}{},
        postheadspace=0.5em, 
        qed=$\blacksquare$, 
        headpunct={},
        postfoothook=\noindent\ignorespaces
    ]{proofstyle}
\declaretheorem[style=proofstyle,numbered=no,name=Proof]{proof}

\declaretheorem[style=theorem,sibling=theorem,name=Lemma]{lemma}
\declaretheorem[style=theorem,sibling=theorem,name=Corollary]{corollary}

\declaretheorem[style=theorem,numbered=no,name=Theorem]{theorem*}
\declaretheorem[style=theorem,numbered=no,name=Lemma]{lemma*}
\declaretheorem[style=theorem,numbered=no,name=Corollary]{corollary*}
\declaretheorem[style=theorem,numbered=no,name=Proposition]{proposition*}
\declaretheorem[style=theorem,numbered=no,name=Claim]{claim*}
\declaretheorem[style=theorem,numbered=no,name=Fact]{fact*}
\declaretheorem[style=theorem,numbered=no,name=Observation]{observation*}
\declaretheorem[style=theorem,numbered=no,name=Conjecture]{conjecture*}

\declaretheorem[style=definition,sibling=theorem,name=Example]{example}

\declaretheorem[style=definition,numbered=no,name=Definition]{definition*}
\declaretheorem[style=definition,numbered=no,name=Remark]{remark*}
\declaretheorem[style=definition,numbered=no,name=Example]{example*}
\declaretheorem[style=definition,numbered=no,name=Question]{question*}

\DeclareMathAlphabet{\mathbfsf}{\encodingdefault}{\sfdefault}{bx}{n}

\DeclareMathOperator*{\conv}{conv}

\usepackage{amsfonts,amsmath,amssymb}
\usepackage{mathtools}
\usepackage{float}
\usepackage{enumitem}
\usepackage{algorithm}
\usepackage{algpseudocode}

\usepackage{times}

\newcommand{\VC}{\operatorname{VC}}

\newcommand{\Risk}{\operatorname{Risk}}
\newcommand{\absop}{\operatorname{abs}}
\renewcommand{\phi}{\varphi}
\usepackage{stmaryrd}  
\newcommand{\pred}[1]{\left\llbracket #1 \right\rrbracket}
\def\d{{\mathrm d}}

\newcommand{\I}{\mathbb{I}}

\newcommand{\R}{\mathbb{R}}

\newcommand{\N}{\mathbb{N}}

\newcommand{\paren}[1]{\left( #1 \right)}

\newcommand{\set}[1]{\left\{ #1 \right\}}

\newcommand{\beq}{\begin{eqnarray*}}
\newcommand{\eeq}{\end{eqnarray*}}
\newcommand{\beqn}{\begin{eqnarray}}
\newcommand{\eeqn}{\end{eqnarray}}
\newcommand{\ben}{\begin{enumerate}}
\newcommand{\een}{\end{enumerate}}
\newcommand{\bit}{\begin{itemize}}
\newcommand{\eit}{\end{itemize}}
\newcommand{\hide}[1]{}

\newcommand{\evalat}[2]{\left.#1\right|_{#2}}

\newcommand{\sgn}{\operatorname{sgn}}

\newcommand{\F}{\mathcal{F}}

\newcommand{\E}{\mathbb{E}}

\newcommand{\X}{\mathcal{X}}
\newcommand{\W}{\mathcal{W}}
\newcommand{\A}{\mathcal{A}}
\newcommand{\Y}{\mathcal{Y}}
\newcommand{\Z}{\mathcal{Z}}
\newcommand{\G}{\mathcal{G}}
\renewcommand{\H}{\mathcal{H}}
\renewcommand{\O}{\mathcal{O}}
\newcommand{\M}{\mathcal{M}}

\newcommand{\fat}{\operatorname{fat}}

\newcommand{\bepf}{\begin{proof}}
\newcommand{\enpf}{\end{proof}}

\usepackage{scalerel}
\usepackage{stackengine}

\newcommand\mywidehat[1]{\ThisStyle{%
  \setbox0=\hbox{$\SavedStyle#1$}%
  \stackengine{-1.0\ht0+.5pt}{$\SavedStyle#1$}{%
    \stretchto{\scaleto{\SavedStyle\mkern.15mu\char'136}{2.6\wd0}}{1.4\ht0}%
  }{O}{c}{F}{T}{S}%
}}

\RequirePackage{mathtools}
\RequirePackage{amsmath}
\RequirePackage{stackrel}
\RequirePackage{dsfont}
\RequirePackage{algorithm}
\RequirePackage{algpseudocode}
\RequirePackage{lipsum}
\RequirePackage{hyperref}

\usepackage{ifpdf}
\ifpdf    %
\usepackage{hyperref}
\else    %
\usepackage[hypertex]{hyperref}
\fi
\usepackage{xspace}

\usepackage[capitalise]{cleveref}

\usepackage{xcolor}

\newcommand{\RE}{\operatorname{RE}}
\newcommand{\AG}{\operatorname{AG}}

\usepackage{multirow}

\title{Improved Generalization Bounds for Adversarially Robust Learning}

\author{
    Idan Attias\thanks{Department of Computer Science, Ben-Gurion University; \texttt{idanatti@post.bgu.ac.il}.} 
    \and Aryeh Kontorovich \thanks{Department of Computer Science, Ben-Gurion University; \texttt{karyeh@cs.bgu.ac.il}.}
    \and Yishay Mansour\thanks{Blavatnik School of Computer Science, Tel Aviv University and Google Research, Tel Aviv; \texttt{mansour.yishay@gmail.com}.}
}

\begin{document}
\maketitle

\begin{abstract}
We consider a model of robust learning in an adversarial
environment. The learner gets uncorrupted training data with access
to possible corruptions that may be affected by the adversary during
testing. The learner's goal is to build a robust classifier, which
will be
tested on future adversarial examples. 
The adversary is limited to $k$ possible corruptions for
each input.
We model the learner-adversary interaction as a zero-sum game.
This model is closely related to the adversarial examples
model of
\citet{schmidt2018adversarially,madry2017towards}.

Our main results consist of generalization bounds for the binary and
multiclass classification, as well as the real-valued case
(regression).
For the binary classification setting, we both tighten
the generalization bound of
\citet{feige2015learning},
and are also
able to handle infinite hypothesis classes.
The sample complexity is improved from
$\O(\frac{1}{\epsilon^4}\log(\frac{|\H|}{\delta}))$ to $\O\big(\frac{1}{\epsilon^2}(k\VC(\H)\log^{\frac{3}{2}+\alpha}(k\VC(\H))+\log(\frac{1}{\delta})\big)$ for any $\alpha > 0$.
Additionally, we extend the algorithm and generalization bound from the binary
to the multiclass and real-valued cases.
Along the way, we obtain results on fat-shattering dimension and Rademacher complexity of $k$-fold maxima
over function classes; these may be of independent interest.

For binary classification, the algorithm of \citet{feige2015learning}
uses a regret minimization algorithm and an ERM oracle as a
black box; we adapt it for the multiclass and regression settings.
The algorithm provides us with near-optimal policies for the players
on a given training sample.
\end{abstract}
\section{Introduction}\label{sec:intro}
We study the classification and regression problems in a setting of
adversarial examples. This setting is different from standard
supervised learning in that examples, at testing
time, may be corrupted in an adversarial manner to disrupt  the
learner's  performance. As standard supervised learning
methods have demonstrated vulnerabilities, the challenge to design reliable robust
models has gained significant attention, and has been termed {\em adversarial
examples}. We study the adversarially robust learning paradigm from
a generalization point of view.

We consider the following robust learning framework for multiclass
and real-valued functions of \citet{feige2015learning}. There is an unknown
distribution over the uncorrupted inputs domain. The learner
receives a labeled uncorrupted sample (the labels can be categorical
or real valued) and has knowledge during the training phase of all
possible corruptions that the adversary might effect. The learner
selects a hypothesis from a fixed hypothesis class (in our case,
a mixture of hypotheses from base class $\H$) that gives a
prediction (a distribution over predictions) for a corrupted input. The learner's accuracy is measured
by predicting the true label of the uncorrupted input while they
observe only the corrupted input during test time. Thus, their goal
is to find a policy that is 
robust against
those corruptions.  The
adversary is capable of corrupting each future input, but there are
only $k$ possible corruptions for 
each 
point in the instance space. 
This suggests the game-theoretic framework of a zero-sum game between the learner and the
adversary.
The model is closely related to the one proposed by
\citet{schmidt2018adversarially,madry2017towards}
and 
other
common robust optimization approaches \citep{BEN:09}, which deal
with bounded worst-case perturbations (under $\ell_\infty$ norm) on the
samples. In this work we do not assume any metric for the
corruptions: the adversary can map an input from the instance space to
any other space, but is limited with finitely many possible corruptions for
each input.

Our main results are generalization bounds for 
classification and regression.
For the binary classification setting, we improve the
generalization bound
given
in \citet{feige2015learning}. 
In particular, we allow for the use of infinite base hypothesis classes $\H$.
The
sample complexity has been improved from
$\O(\frac{1}{\epsilon^4}\log(\frac{|\H|}{\delta}))$ to \\
$\O\big(\frac{1}{\epsilon^2}(k\VC(\H)\log^{\frac{3}{2}+\alpha}(k\VC(\H))+\log(\frac{1}{\delta})\big)$, for any $\alpha > 0$.
Roughly speaking, the core of all proofs is a bound on the Rademacher complexity of the
$k$-fold maximum of the convex hull of the loss class of $\H$. The
$k$-fold maximum captures the $k$ possible corruptions for each
input. 
In the regression setting we provide three different generalization bounds. One of the main contributions in this setting is an upper bound on the empirical fat-shattering dimension of $k$-fold maximum class.

Our algorithm is an adaptation of the 
regret minimization algorithm
proposed
for binary classification
by \citet{feige2015learning} for computing
near optimal-policies for the players on the training data
to the 
multiclass classification 
settings.
It
is a variant of the algorithm found in \citet{cesa2007improved}
and based on the ideas of \citet{freund1999adaptive}. An ERM (empirical risk minimization) oracle is 
repeatedly used
to return a hypothesis from a fixed hypothesis class $\H$ that minimizes the error rate on a given sample, while weighting samples differently every time.
The learner uses a randomized classifier chosen uniformly from the mixture of hypotheses returned by the algorithm.

Thus, we extend the
ERM paradigm by using \textit{adversarial training} techniques instead of
merely find a hypothesis that minimizes the empirical risk. In contradistinction to
``standard'' learning, ERM often does not yield models that are
robust to adversarially corrupted examples
\citep{szegedy2013intriguing,biggio2013evasion,goodfellow2014explaining,kurakin2016adversarial,moosavi2016deepfool,tramer2017space}. 

\subsection{Subsequent Work: \citet*{montasser2019vc,montasser2020reducing} }\label{subsec:subsequent_work}
Following the 
conference version
\citep{attias2019improved}
of this work,
\citet*{montasser2019vc} have proved that $\VC$ classes are robustly PAC-learnable only improperly (that is, the hypothesis is selected from a broader class than that of the true concept), with respect to any arbitrary perturbation set, possibly of infinite size. The sample complexity\footnote{$\Tilde{\O}(\cdot)$ hides 
poly-logarithmic factors of $\VC,\VC^*,1/\epsilon,1/\delta$.} is independent of the number of allowed perturbations, $\tilde{\O}\Big(\frac{\VC(\H)\VC^*(\H)}{\epsilon}+\frac{1}{\epsilon}\log\frac{1}{\delta}\Big)$ in the realizable setting and $\tilde{\O}\Big(\frac{\VC(\H)\VC^*(\H)}{\epsilon^2}+\frac{1}{\epsilon^2}\log\frac{1}{\delta}\Big)$ in the agnostic setting, where $\VC^*(\H) $ denotes the dual $\VC$-dimension. Their approach relies on sample compression arguments whereas uniform convergence does not hold. 
As a by-product, for the case of $k<\infty$ possible corruptions for each input, they obtained a sample complexity of size $\O\left(
\frac{\VC(\H)\log k}{\epsilon^2}\text{polylog}(\frac{\VC(\H)\log k}{\epsilon}) +\frac{1}{\epsilon^2}\log(\frac{1}{\delta})
\right)$ for the zero-one robust loss (which is defined below).

The main difference of between the two works is the definition of the loss function.
Specifically, for functions $h_1,\ldots,h_T$, in the binary classification setting, we define the loss $\ell : \Delta(\H)\times\X\times\Y \rightarrow [0,1]$ by 
\begin{align}\label{def:continuous-loss}
\ell_1(h_1,\ldots,h_T,x,y)= \underset{z\in \rho(x)}{\max} \frac{1}{T} \sum_{i=1}^T \I\left[h_i(z)\neq y\right]=\underset{z\in \rho(x)}{\max} \left |\frac{1}{T} \sum_{i=1}^T  h_i(z)- y \right|,
\end{align}
which we refer to as the $[0,1]$-\textit{robust loss}.  \citet{montasser2019vc,montasser2020reducing}
defined a loss function $\ell : \H\times\X\times\Y \rightarrow \set{0,1}$ as follows
\begin{align}\label{def:zero-one-loss}
\ell_2(h,x,y)=\underset{z\in \rho(x)}{\max} \I\left[h(z)\neq y\right],
\end{align}
which we refer to as the \textit{zero-one robust loss}.
More specifically, they consider for functions $h_1,\ldots,h_T$ the loss
\begin{align}
\ell_3(h_1,\ldots,h_T,x,y)=\underset{z\in \rho(x)}{\max} \I\left[\mbox{\textrm Majority}(h_1(z), \ldots , h_T(z))\neq y\right],
\end{align}
where \textrm{Majority} takes the majority of its the Boolean inputs (and assume that $T$ is odd). Clearly, if \\$\ell_1(h_1,\ldots,h_T,x,y) < 1/2$ then $\ell_3(h_1,\ldots,h_T,x,y)=0$. However, if $\ell_3(h_1,\ldots,h_T,x,y)=0$ it only guarantees that $\ell_1(h_1,\ldots,h_T,x,y)<1/2$ but can be very far from zero. This is why an upper bound on sample complexity of $\ell_1 $ implies an upper bound on the sample complexity of $\ell_3$, but not vice versa.
We summarize the main results for both definitions in \cref{subsec:results-sum}.

The work of \citet{montasser2019vc}, that considers the zero-one robust loss, improper learning is necessary due to the lack of uniform convergence, which may arise in the case of infinite set of corruptions. The learner competes with the single optimal hypothesis in the class, and outputs a mixture of hypothesis to do so.
In this work, considering the  $[0,1]$-robust loss, 
we would like to guarantee and $\epsilon$-optimal value for the learner in a zero-sum game, via a mixed strategy, and so we find an $\epsilon$-optimal mixture of hypothesis. That is, we compete with the optimal mixture of hypothesis from the function class. In that sense, we are having a proper learning algorithm, with respect to the convex hull of the hypothesis class.

In another closely related work from the computational perspective, \citet{montasser2020reducing} reduced the problem of robust learning to non-robust learning. Namely, their algorithm using access to only a black-box PAC learner, similar to the algorithm of \citet{feige2015learning} that we employ in this paper. They provided an algorithm that achieves small robust risk in the realizable setting with sample complexity (that is independent of $k$) of $\tilde{\O}\Big(\frac{\VC(\H)\left(\VC^{*}(\H)\right)^2}{\epsilon}+\frac{1}{\epsilon}\log\frac{1}{\delta}\Big)$, 
and uses $\O\paren{\log^2(nk)+\log\frac{1}{\delta}}$ black-box oracle calls to any PAC-learner, where $n$ is the sample size.
Their result relies on sample compression and not uniform convergence.

\subsection{Uniform Convergence of the Zero-One Robust Loss Class} 
For the case of finite set of corruptions, and learning with respect to the zero-one robust loss, we show that the $\VC$ dimension of the robust loss class remains finite (as opposed to the case of infinite corruptions). As a result, we have uniform convergence, and robust $\mathrm{ERM}$ suffices to ensure learning. (The proof is in \cref{app:gen-bound-binary-additional}).
\begin{lemma}\label{lem:vc-robust-loss}
 For any class $\H$ of $\VC$ dimension $d$, and any adversary $\rho: \X \rightarrow 2^\X$ such that  $|\rho(x)|\leq k$, the $\VC$-dimension of the zero-one robust loss class $L^\rho_{\H}=\set{(x,y)\mapsto \underset{z\in \rho(x)}{\max}\I\left[h(z)\neq y\right]: h \in \H}$ is at most $\O(d \log k)$.
\end{lemma}
Via a standard uniform convergence argument, we have the following result.
\begin{theorem}
 For any class $\H \subseteq \{0,1\}^{\X}$ of $\VC$ dimension $d$, and any adversary $\rho: \X \rightarrow 2^{\X}$ such that  $|\rho(x)|\leq k$. For the robust zero-one loss function $\ell(h,x,y) = \underset{z\in \rho(x)}{\max} \I[h(z)\neq y]$, the sample complexity for the realizable setting is $\M_{\RE}(\epsilon,\delta,\H,\rho)= \O\paren{\frac{d\log k}{\epsilon}\log\frac{1}{\epsilon}+\frac{1}{\epsilon}\log\frac{1}{\delta}}$,
and the sample complexity for the agnostic setting is 
$\M_{\AG}(\epsilon,\delta,\H,\rho)=\O\paren{\frac{d\log k}{\epsilon^2}+\frac{1}{\epsilon^2}\log\frac{1}{\delta}}.$
\end{theorem}

\subsection{Main Results}\label{subsec:results-sum}
We provide a summary of the results for the $[0,1]$-robust loss and the zero-one robust loss (see \cref{def:continuous-loss,def:zero-one-loss} for the definitions) for robust $(\epsilon,\delta)$-PAC learning with finite set of possible corruptions.
\paragraph{Notation.} $d$ denotes the $\VC$ dimension, $d^*$ denote the dual dual-$\VC$ dimension, $\fat_{\gamma}(\cdot)$ is the $\gamma-$fat shattering dimension, and $k$ is the size of possible corruptions for each input. $\Tilde{\O}(\cdot)$ stands for for omitting poly-logarithmic factors  of $d,d^*,1/\epsilon,1/\delta$.

\begin{table}[H]
    \begin{center}
    \bgroup
    \setlength{\arrayrulewidth}{0.2mm} 
    \setlength{\tabcolsep}{4pt} 
    \def\arraystretch{1.5} 
        \begin{tabular}{|c|c||c|}
            \hline
            \multicolumn{3}{|c|}{\textbf{Sample complexity for agnostic learning with $[0,1]$-robust loss}}
            \\
            \hline
            \hline
            \textsc{Generalization} & \textsc{Binary Classification} & \textsc{Reference}
            \\
            \hline
            Uniform Convergence
            &
            $\O\paren{\frac{1}{\epsilon^4}\log\frac{|\H|}{\delta}}$
            &\citet{feige2015learning}
             \\
            \hline
            \multirow{2}{9em}{\centering{Sample Compression}}
            &
             $\tilde{\O}\Big(\frac{dd^*}{\epsilon^4}+\frac{1}{\epsilon^4}\log\frac{1}{\delta}\Big)$
             & \multirow{2}{9em}{\centering{\citet{montasser2019vc}}}
            \\
            &
            $\tilde{\O}\Big(\frac{d\log{k}}{\epsilon^4}+\frac{1}{\epsilon^4}\log\frac{1}{\delta}\Big)$
            &
            \\
            \hline
            Uniform Convergence
            &
            $\tilde{\O} \Big(\frac{kd}{\epsilon^2}+\frac{1}{\epsilon^2}\log\frac{1}{\delta}\Big)$
            & This work
            \\
            \hline
            \hline
            &
            \textsc{Regression}
            &
            \\
            \hline
            Uniform Convergence
            &
            $\tilde{\O}\left(\inf_{\beta\geq 0}\set{\beta + \sqrt{\frac{k}{n}}\int_{\beta}^{1} \sqrt{\fat_{\gamma}(\H)}d\gamma } 
            + \sqrt{\frac{\log\left(\frac{1}{\delta}\right)}{n}}
            \right)$
            &
            This work
            \\
            \hline
        \end{tabular}
    \egroup
    \end{center}
\end{table}

\begin{table}[H]
    \begin{center}
    \bgroup
    \setlength{\arrayrulewidth}{0.2mm} 
    \setlength{\tabcolsep}{1.5pt} 
    \def\arraystretch{1.6} 
        \begin{tabular}{|c|c|c||c|}
            \hline
            \multicolumn{4}{|c|}{\textbf{Sample complexity for binary classification with zero-one robust loss}}
            \\
            \hline
            \hline
            \textsc{Generalization} & \textsc{Realizable} & \textsc{Agnostic} & \textsc{Reference}
            \\
            \hline
             \multirow{2}{9em}{\centering{Sample Compression}}
             & $\tilde{\O}\Big(\frac{dd^*}{\epsilon}+\frac{1}{\epsilon}\log\frac{1}{\delta}\Big)$ 
             & $\tilde{\O}\Big(\frac{dd^*}{\epsilon^2}+\frac{1}{\epsilon^2}\log\frac{1}{\delta}\Big)$
             & \multirow{2}{10em}{\centering{\citet{montasser2019vc}}}
             \\
             &$\tilde{\O}\paren{\frac{d\log k}{\epsilon}+\frac{1}{\epsilon}\log\frac{1}{\delta}}$  
             & $\tilde{\O}\paren{\frac{d\log k}{\epsilon^2}+\frac{1}{\epsilon^2}\log\frac{1}{\delta}}$ 
             &
             \\
            \hline
             Uniform Convergence
             & 
             $\O \paren{\frac{d\log k}{\epsilon}\log\frac{1}{\epsilon}+\frac{1}{\epsilon}\log\frac{1}{\delta}}$ 
             & 
             $\O\paren{\frac{d\log k}{\epsilon^2}+\frac{1}{\epsilon^2}\log\frac{1}{\delta}}$
             & This work
             \\
            \hline
        \end{tabular}
    \egroup
    \end{center}
\end{table}

Whether we can achieve a sample complexity of $\approx \frac{d\log k}{\epsilon^2} \;\text{or}\; \frac{dd^*}{\epsilon^2}$ for agnostic learning with the $[0,1]$-robust loss remains an open question. 
The method of \citet{montasser2019vc} can be modified to accommodate learning with respect to the $[0,1]$ robust loss. Specifically, taking the majority of weak learners is not sufficient for obtaining an $\epsilon$-optimal mixed strategy. Rather, we take a majority of strong learners (each with $\epsilon$ error), each of which takes $\approx \frac{d}{\epsilon^2}$ samples (and not $\approx d$). This implies a sample complexity (via sample compression scheme) of $\frac{dd^*}{\epsilon^4}$ or $\frac{d\log(k)}{\epsilon^4}$.

\subsection{Other Related Work}
The most closley related works studying robust learning with adversarial examples are \citet{schmidt2018adversarially,madry2017towards}. Their model deals with bounded worst-case perturbations (under $\ell_\infty$ norm) on the samples. This is slightly different from our model as we mentioned above. 
Other related works that analyze the theoretical aspects of adversarial robust generalization are \citet{montasser2019vc,yin2019rademacher,awasthi2020rademacher,cullina2018pac,khim2018adversarial,raghunathan2019adversarial,diochnos2018adversarial,balda2019adversarial,pydi2019adversarial,tu2019theoretical,chen2020more,carmon2019unlabeled,alayrac2019labels,zhai2019adversarially,najafi2019robustness,levi2021domain,attias2022characterization,attias2022adversarially}. A different notion of robustness by \citet{xu2012robustness} is shown to be sufficient and necessary for standard generalization.
Learning with adversarial examples is extensively studied from the computational point of view as well \citep{bubeck2018adversarial,mahloujifar2019curse,mahloujifar2019can,chen2017robust,awasthi2019robustness,awasthi2019adversarially,sinha2017certifying,diakonikolas2019nearly,diakonikolas2020complexity,montasser2020efficiently,gourdeau2019hardness,ashtiani2020black}.

All of our results based on a robust learning model for binary
classification suggested by \citet{feige2015learning}. The works of
\citet{mansour2014robust,feige2015learning,feige2018robust} consider {\em
robust inference} for the binary and multiclass case. The robust
inference model assumes that the learner knows both the distribution
and the target function, and the main task is given a corrupted
input, derive in a computationally efficient way a classification
which will minimize the error. In this work we consider only the
learning setting, where the learner has only access to an
uncorrupted sample, and need to approximate the target function on
possibly corrupted inputs, using a restricted hypothesis class $\H$.

The work of \citet{globerson2006nightmare} and its
extensions \citet{teo2008convex,dekel2010learning}
discuss a robust learning model where an uncorrupted sample is drawn
from an unknown distribution, and the goal is to learn a linear
classifier 
resilient against
missing attributes in
future test examples. They discuss both the static model (where the
set of missing attributes is selected independently from the
uncorrupted input) and the dynamic model (where the set of missing
attributes may depend on the uncorrupted input). The model we use
\citep{feige2015learning} extends the robust learning model to handle
corrupted inputs (and not only missing attributes) and an arbitrary
hypothesis class (rather than only linear classifiers).

There is a vast literature in statistics, operation research and
machine learning regarding various noise models. Typically, most
noise models assume a random process that generates the noise. In
computational learning theory, popular noise models include random
classification noise \citep{angluin1988learning} and malicious noise
\citep{valiant1985learning,kearns1993learning}. In the
malicious noise model, the adversary gets to arbitrarily corrupt
some small fraction of the examples; in contrast, in our model the
adversary can always corrupt every example, but only in a limited
way.

\section{Model} \label{sec:model}
There is an unknown distribution $D$ over 
some
domain $\X$ of uncorrupted examples and a finite domain of corrupted examples $\Z$, possibly the same as $\X$. 
Our setting is the \textit{agnostic PAC-learning} framework in a deterministic scenario.
The label of each input is uniquely determined by an arbitrary unknown target function $c:\X \rightarrow \Y$. The function
 $c$ maps each uncorrupted input $x\in \X$ to a label $c(x)=y$, where the set of labels $\Y$ can be $\set{1,\dots,l}$ or $\R$.

The adversary is able to corrupt an input by mapping an uncorrupted input $x\in \X$ to a corrupted one
$z\in \Z$. There is a mapping $\rho$ which for every $x \in \X$ defines a set $\rho(x) \subseteq \Z$, such that $|\rho(x)| \leq k$. The adversary can map an uncorrupted input $x$ to any corrupted input $z \in \rho(x)$. We assume that the learner has an access to $\rho(\cdot)$ during the training phase.

There is a fixed hypothesis class $\H$ of hypothesis $h:\Z \mapsto \Y$ over corrupted inputs. The learner observes an uncorrupted sample $S_u=\{\langle x_1,c(x_1) \rangle, \ldots, \langle x_m, c(x_m) \rangle\}$, where $x_i$ is drawn i.i.d. from $D$,
and selects a mixture of hypotheses from $\H$, $\Tilde{h}\in \Delta(\H)$. In the classification setting, $\Tilde{h}: \Z \rightarrow \Delta{(\Y)}$ is a mixture $\set{h_i|\H \ni h_i:\Z \rightarrow \Y}_{i=1}^T$ such that label $y\in\Y=\set{1,\dots,l}$ gets a mass of $\sum_{i=1}^T\alpha_i\I\left[h_i(z)= y\right]$ where $\sum_{i=1}^T\alpha_i$=1. For each hypothesis $h\in\H$ in the mixture we use the zero-one loss to measure the quality of the classification, i.e., $\ell(h(z),y)=\I\left[h(z)\neq y\right]$. The loss of $\Tilde{h}\in\Delta{(\H)}$ is defined by $\ell(\Tilde{h}(z),y)= \sum_{i=1}^T \alpha_i \ell(h_i(z),y)$. In the regression setting, $\Tilde{h}: \Z \rightarrow \R$ is a mixture $\set{h_i|\H \ni h_i:\Z \rightarrow \R}_{i=1}^T$ and is defined by $\Tilde{h}(z)=\sum_{i=1}^T \alpha_i h_i(z)$. For each hypothesis $h\in\H$ in the mixture we use $L_1$ and $L_2$ loss functions, i.e., $\ell(h(z),y)=|h(z)-y|^p$, for $p=1,2$. We assume the $L_1$ loss is bounded by 1. Again, the loss of $\Tilde{h}\in\Delta{(\H)}$ is defined by $\ell(\Tilde{h}(z),y)= \sum_{i=1}^T \alpha_i \ell(h_i(z),y)$.

The test phase proceeds
as follows. 
First, an uncorrupted input $x\in \X$ is 
drawn from
$D$. Then, the adversary selects $z \in \rho(x)$, given $x\in\X$.
The learner observes a corrupted input $\Z$ and outputs a prediction, as dictated
by $\Tilde{h}\in\Delta(\H)$. Finally, the learner incurs a loss as described above. The main difference from the classical learning models is that the learner will be tested on adversarially corrupted inputs $z \in \rho(x)$. When selecting a strategy this needs to be taken into consideration.

The goal of the learner is to minimize the expected loss, while the adversary would like to maximize it. This defines a zero-sum game which has a value $v$ which is the learner's error rate. We say that the learner's hypothesis is $\epsilon$-optimal if it guarantees a loss which is at most $v+\epsilon$, and the adversary policy is $\epsilon$-optimal if it guarantees a loss which is at least $v-\epsilon$. We refer to a 0-optimal policy as an optimal policy.

Formally, the error (risk) of the learner when selecting a hypothesis $\Tilde{h}\in\Delta(\H)$ is
\beq
\Risk(\Tilde{h})=\E_{x\sim D}[\max_{z\in\rho(x)}\ell(\Tilde{h}(z),c(x))],
\eeq
and their goal is to choose $\Tilde{h}\in\Delta(\H)$ with an error close to
\beq
\min_{\Tilde{h}\in\Delta(\H)}\Risk(\Tilde{h})=\min_{\Tilde{h}\in\Delta(\H)}\E_{x\sim D}[\max_{z\in\rho(x)}\ell(\Tilde{h}(z),c(x))]= v 
.
\eeq

\section{Definitions and Notation} \label{sec:defs}
For a function class $\H$ with domain $\Z$ and range $\Y=\set{1,\dots,l}$, denote the zero-one loss class
\beq
L_{\H}:=\set{Z\times\set{1,\dots,l}\ni(z,y)\mapsto \I\left[h(z)\neq y\right]: h\in \H}
.
\eeq
For $\H$ with domain $\Z$ and range $\R$, denote the $L_p$ loss class
\beq
L_{\H}^p:=\set{Z\times\R\ni(z,y)\mapsto |h(z)-y|^p: h\in \H}
.
\eeq

Throughout the article, we assume a bounded loss $\ell(h(z),y)\leq M$. Without the loss of generality we use $M=1$, 
since
otherwise, $M$ can be re-scaled.

We define the operator $\conv$ as the convex hull of a real-valued function class,
\beq
\conv(\F) := \set{ W \ni w\mapsto
  \sum_{t=1}^T \alpha_t f_t(w):
  T\in\N, \alpha_t\in[0,1], \sum_{t=1}^T\alpha_t=1, f_t\in \F}
.
\eeq
We also define the convex hull of loss class $L$, where the data is corrupted by $\rho(\cdot)$,
\beq
\conv^{\rho}(L) := \set{ \X\times \Y\ni (x,y)\mapsto \max_{z\in\rho(x)} \sum_{t=1}^T \alpha_t f_t(z,y):
  T\in\N,
  \alpha_t\in[0,1], \sum_{t=1}^T \alpha_t=1, f_t\in L}.
\eeq
For $1\leq j\leq k$ define,
\begin{align} \label{def:jth-loss-class}
\F_\H^{(j)}:=\set{\X\times \Y \ni (x,y)\mapsto \I\left[h(z_j)\neq y\right]:h\in \H,\;\rho(x)=\set{z_1,\dots,z_k}},
\end{align}
where we treat the set-valued output of $\rho(x)$ as an ordered list,
and $\F_\H^{(j)}$ is constructed by taking the $j$th element in this list, for each input $x$.

For a set $W$ and $k$ function classes $\A^{(1)},\ldots,\A^{(k)}\subseteq \R^W$, define the
$\max$ operator
\beq
\max
\paren{(\A^{(j)})_{j\in[k]}} :=
\set{ W\ni w\mapsto\max_{j\in[k]}f^{(j)}(w): f^{(j)}\in \A^{(j)}}.
\eeq
The composition of $\max$ and $\conv$ operators
$\max\paren{(\conv(\A^{(j)}))_{j\in[k]}}$
is well-defined, 
note that 
\begin{align}
\label{eq:conv-rho-max}
\conv^{\rho}(L_\H) 
\subseteq
\max\paren{(\conv(\F_\H^{(j)}))_{j\in[k]}}
.    
\end{align}

Denote the error (risk) of hypothesis $h:\Z \mapsto \Y$ under corruption of $\rho(\cdot)$
by
\begin{align*}
\Risk(h)=\E_{x\sim D}[\max_{z\in\rho(x)}\ell(h(z),c(x))],
\end{align*}
and the empirical error on sample $S$ under corruption of $\rho(\cdot)$ by
\begin{align*}
\mywidehat{\Risk}(h)=\frac{1}{|S|}\underset{(x,y)\in S}{{\displaystyle \sum}}\max_{z\in\rho(x)}\ell(h(z),c(x)).
\end{align*}

\subsection{Combinatorial Dimensions and Capacity Measures}
\paragraph{Rademacher Complexity.} Let $\H$ be of real valued function class on the domain space $\W$. Define the empirical Rademacher complexity on a given sequence $\mathbf{w}=(w_1,\ldots,w_n)=w_{1:n}\in \W^n$:
\beq
R_n(\H | \mathbf{w})
=
E_{\mathbf{\sigma}} \sup_{h\in H}\frac1n\sum_{i=1}^n\sigma_i h(w_i).
\eeq
\paragraph{Fat-Shattering Dimension.} For $\F\subset\R^\X$ and $\gamma>0$,
we say that $\F$ $\gamma$-shatters
a set $S=\set{x_1,\ldots,x_m}\subset\X$
if there exists an $r=(r_1,\ldots,r_m)\in\R^m$ such that for each $b\in\set{-1,1}^m$
there is a function $f_b\in\F$ such that

\beq
\forall i\in [m] : \left.
\begin{cases}
    f_b(x_i)\ge r_i+\gamma & \text{if } b_i=1\\
    f_b(x_i)\le r_i-\gamma & \text{if } b_i=-1
\end{cases}
\right..
\eeq
We refer to $r$ as the {\em shift}.
The $\gamma$-fat-shattering dimension, denoted by $\fat_\gamma(\F)$,
is the size of the largest $\gamma$-shattered set (possibly $\infty$).

\paragraph{Graph Dimension.} Let $\H \subseteq \Y^\X$ be a categorical function class such that $\Y=[l]=\set{1\dots,l}$. Let $S \subseteq \X$. We say that $\H$ $G$-shatters $S$ if there exists an $f : S \mapsto \Y$ such that for every $T \subseteq S$ there is a $g \in \H$ such that
\beq
\forall x \in T,\; g(x)=f(x)\; \text{and}\;\; \forall x\in S\setminus T,\; g(x)\neq f(x).
\eeq
The graph dimension of $\H$, denoted $d_G(\H)$, is the maximal cardinality of a set that is $G$-shattered by $\H$.

\paragraph{Natarajan Dimension.} Let $\H \subseteq \Y^\X$ be a categorical function class such that $\Y=[l]=\set{1\dots,l}$. Let $S \subseteq \X$. We say that $\H$ $N$-shatters $S$ if there exist $f_1, f_2 : S \mapsto \Y$ such that for every $y \in S$ $f_1(y) \neq f_2(y)$, and for every $T \subseteq S$ there is a $g\in \H$ such that
\beq
\forall x\in T,\; g(x)=f_1(x),\; \text{and}\;\; \forall x\in S\setminus T,\; g(x)=f_2(x).
\eeq
 The Natarajan dimension of $\H$, denoted $d_N(\H)$, is the maximal cardinality of a set that is $N$-shattered by $\H$.

\paragraph{Growth Function.} The growth function $\Pi_\H:\N \mapsto \N$ for a binary function class $\H:\X \mapsto \set{0,1}$ is defined by
\beq
\forall m\in \N,\;\Pi_\H(m)= \max_{\set{x_1,\dots,x_m}\subseteq \X}|\set{(h(x_1),\dots,h(x_m)):h\in \H}|
\eeq
And the $\VC$-dimension of $\H$ is defined by
\beq
\VC(\H)=\max\set{m:\Pi_\H(m)=2^m}.
\eeq
\section{Algorithm}\label{sec:algo}
We have a base hypothesis class $\H$ with domain $\Z$ and range $\Y$ that can be $\set{1,\dots,l}$ or $\R$.
The learner receives a labeled uncorrupted sample and has access during the training to possible corruptions by the adversary.
We
employ the
regret minimization algorithm proposed by \citet{feige2015learning} for
binary classification,
and extend it to the regression and multiclass classification settings.

A brief description
of the algorithm is as follows. Given $x\in\X$, we define a $|\rho(x)|\times\H$
loss matrix $M_x$ such that $M_x(z,h)=\I\left[h(z)\neq y\right]$,
where $y=c(x)$. The learner's strategy is a distribution $Q$ over $\H$. The adversary's strategy $P_x\in\Delta(\rho(x))$, for a given $x\in\X$, is a distribution  over the corrupted inputs $\rho(x)$. We can treat $P$ as a vector of distributions $P_x$ over all $x\in\X$. Via the minimax principle, the value of the game is
\beq
v=\min_{Q}\max_{P}\E_{x\sim D}[P_x^TM_xQ]=\max_{P}\min_{Q}\E_{x\sim D}[P_x^TM_xQ]
\eeq
For a given $P$,
a learner's
minimizing $Q$ is simply a hypothesis that minimizes the error when
the distribution over pairs $(z, y) \in \Z \times \Y$ is $D^P$, where
\beq
D^P(z,y)= \sum_{x:\;c(x)=y \wedge z\in \rho(x)}P_x(z)D(x)
.
\eeq
Hence, the learner selects
\beq
h^P =\arg\min_{h\in \H}\E_{(z,y)\sim D^P}[\ell(h(z),y)]
.
\eeq
A hypotheses $h^P$ can be found using the ERM oracle, when $D^P$ is the empirical distribution over a training sample. 

Repeating this process multiple times yields a mixture of hypotheses $\Tilde{h}\in\Delta(\H)$ (mixed strategy- a distribution $Q$ over $\H$) for the learner. The learner uses a randomized classifier chosen uniformly from this mixture. This also yields a mixed strategy for the adversary, defined by an average of vectors $P$. Therefore, for a given $x\in\X$, the adversary uses a distribution $P_x\in\Delta(\rho(x))$ over corrupted inputs. 

\begin{algorithm}[H]
\caption{}
\begin{algorithmic}[1]
\Statex \textbf{parameter}: $\eta>0$
\ForAll {\texttt {$(x,y)\in S, z\in \rho(x)$}} \Comment{initialize weights and distributions vector}
    \State $w_1(z,(x,y)) \gets 1$, $\forall (x,y)\in S, \forall z\in \rho(x)$
    \State $P^{1}(z,(x,y)) \gets \frac{w_{1}(z,(x,y))}{\sum_{z'\in \rho(x)}w_{1}(z',(x,y))}$ \Comment for each $(x,y)\in S$ we have a distribution over $\rho(x)$
\EndFor
\For{\texttt{t = 1:T}}
        \State \texttt{$h_t \gets \arg\underset{h\in \H}\min E_{(z,y)\sim D^{P^t}}[\ell(h(z),y] $}\Comment using the ERM oracle for $\H$
\ForAll {\texttt {$(x,y)\in S, z\in \rho(x)$}} \Comment{update weights for $P^{t+1}$}
        \State $w_{t+1}(z,(x,y)) \gets (1+\eta\cdot[\ell(h_t(z),y)])\cdot w_t(z,(x,y)) $
        \State $P^{t+1}(z,(x,y)) \gets \frac{w_{t+1}(z,(x,y))}{\sum_{z'\in \rho(x)}w_{t+1}(z',(x,y))}$
\EndFor
\EndFor
\State \textbf{return} $h_1,\dots,h_T$ for the learner, $\frac{1}{T}\sum_{t=1}^T P^t$ for the adversary
\end{algorithmic}
\end{algorithm}
\noindent Similar to \citet[Theorem 1]{feige2015learning}, for the binary
classification case and zero-one loss we have:
\begin{theorem}\citep*[Theorem 1]{feige2015learning}
Fix a sample S of size $n$, and let $T \geq \frac{4n\log k}{\epsilon^2}$, where $k$ is the number of possible corruptions for each input. For an uncorrupted sample S we have that the strategies $P =\frac{1}{T}\sum_{t=1}^T P^t$ for the adversary and $h_1,\dots,h_T$ (each one of them chosen uniformly) for the learner are $\epsilon$-optimal strategies on S.
\end{theorem}
Assuming a bounded loss, i.e.,  $\ell(h(z),y)\leq 1\; ,\forall x\in \X,\forall z\in \Z,\forall h \in \H$, the result remains the same for the other settings.

\section{Generalization Bound for Classification} \label{sec:classification-bound}

We would like to show that if the sample $S$ is large enough, then
the policy achieved by the algorithm above will generalize well. We
both improve a generalization bound, previously found in
\citet{feige2015learning}, which handles any mixture of hypotheses
from $\H$, and also are able to handle an infinite hypothesis class
$\H$.
The sample complexity is improved from
$\O(\frac{1}{\epsilon^4}\log(\frac{|\H|}{\delta}))$ to $\O\big(\frac{1}{\epsilon^2}(k\VC(\H)\log^{\frac{3}{2}+\alpha}(k\VC(\H))+\log(\frac{1}{\delta})\big)$for any $\alpha > 0$.
\begin{theorem}[Generalization bound for binary classification]
\label{thm:binary-bound}
Let $\H:\Z \mapsto \set{0,1}$ be a hypothesis class with finite $\VC$-dimension.
For any $\alpha > 0$ there exists a constant $C_\alpha$ and there is a sample complexity $n_0=\frac{C_\alpha}{\epsilon^2}\left(k\VC(\H)\log^{\frac{3}{2}+\alpha}(k\VC(\H))+\log(\frac{1}{\delta})\right)$, such that for $|S|\geq n_0$, for every  $\Tilde{h}\in\Delta(\H)$
\beq
|\Risk(\Tilde{h})-\mywidehat{\Risk}(\Tilde{h})|\leq \epsilon
\eeq
with probability at least $1-\delta$.
\end{theorem}

\begin{theorem}\citep[Theorem 3.3]{mohri2018foundations} \label{thm:rad-gen-bound} Let $\G$ be a family of functions mapping from
$\W$ to ${[}0,1{]}$. 
Then, for any $\delta > 0$ with probability at least $1 - \delta$ over the draw of an i.i.d. sample $S=(w_1,\cdots,w_n)=\mathbf{w}\;$  from distribution $D$, for all $g \in \G$:\\
\beq
E_{w\sim D}[g(w)]-\frac{1}{n}\sum_{i=1}^n g(w_{i}) 
\leq 
2R_n(\G |\mathbf{w})+3\sqrt{\frac{\log\left(\frac{2}{\delta}\right)}{2n}}.
\eeq
\end{theorem}
\paragraph{Remark.}
The corresponding result in the conference version of this paper,
\citet{attias2019improved}, Theorem 2, was proved via Lemma 3 therein.
The latter contained a mistake, as pointed out to us
by
Digvijay Pravin Boob
and
Praneeth Netrapalli.
The current proof relies on a recent result of \citet{foster2019ell_}.

\begin{theorem}\citep{foster2019ell_}\label{thm:contraction} Let $\F$ be a $\R^k$-valued function class, such that the coordinate projection class is denoted by $\F_j = \{w\mapsto f(w)_j \, | \, f\in \F \}$, for $1\leq j \leq k$. 
Let $(\phi_t)_{t\leq n}$ be a sequence of functions such that each $\phi_t$ is $L$-Lipschitz with respect to $\ell_\infty$ norm.
For any $\alpha >0$, there exists a constant $C_{\alpha}>0$ such that if $|\phi_t(f(w))| \lor ||f(w)||_{\infty} \leq B$,
then it holds for any sequence $\mathbf{w}=(w_1,\cdots,w_n)$,
\begin{align*}
R_n(\phi \circ \F  | \mathbf{w})
&:=
E_{\mathbf{\sigma}} \sup_{f\in F}\frac1n\sum_{t=1}^n\sigma_t \phi_t(f_t(w_t)) 
\\
&\leq
C_{\alpha}L\sqrt{k}\cdot \max_{i\in [k]}\sup_{\mathbf{a}=(a_1,\dots,a_n)}R_n(\F_i|\mathbf{a})\cdot \log^{\frac{3}{2}+\alpha}\Bigg(\frac{Bn}{\max_{i\in [k]}\sup_{\mathbf{a}=(a_1,\dots,a_n)}R_n(\F_i|\mathbf{a})}\Bigg).
\end{align*}
\end{theorem}

\begin{proof}[Proof of Theorem \ref{thm:binary-bound}].
Our strategy is to bound the empirical Rademacher complexity (over the sample points) of the loss class of $\Tilde{h}\in\Delta(\H)$. 
As we mentioned in \cref{eq:conv-rho-max}, 
$\conv^{\rho}(L_\H) 
\subseteq
\max(\conv(\F_\H^{(j)}))_{j\in[k]})$.
Recall that functions contained in $\F_\H^{(j)}$ are loss functions of the learner when the adversary corrupts input $x$ to $z_j\in \rho(x)$.
We are left to bound the Rademacher complexity of the function class 
$\max((\conv(\F^{(j)}_{\H}))_{j\in[k]})$.
Formally,
\beq
|\Risk(\Tilde{h})-\mywidehat{\Risk}(\Tilde{h})|
&=& 
|E_{(x,y)\sim D}\max_{j\in[k]}\sum_{t=1}^T \alpha_t f_t^{(j)}(x,y)-\frac{1}{n}\underset{(x,y)\in S}{{\displaystyle \sum}}\max_{j\in[k]}\sum_{t=1}^T \alpha_t f_t^{(j)}(x,y)| 
\\
&\le&
2R_n\paren{\max((\conv(\F^{(j)}_{\H}))_{j\in[k]})|\mathbf{x}\times \mathbf{y}}+3\sqrt{\frac{\log\left(\frac{2}{\delta}\right)}{2n}},
\eeq
where 
the inequality stems from applying Theorem \ref{thm:rad-gen-bound} 
on the function class $\conv^{\rho}(L_\H)$ and \cref{eq:conv-rho-max}.
By taking $\phi(z_1,\cdots, z_k)=\max_{j\in [k]}z_j$, which is a $1$-Lipschitz with respect to $\ell_\infty$, and  $\F = \{(x,y)\mapsto (f_1(x,y),\cdots, f_k(x,y))  \, | \, f_j\in \conv(\F_\H^{(j)}),1\leq j \leq k \}$ we can apply Theorem \ref{thm:contraction}, for any $\alpha >0$, there exists a constant $C_\alpha>0$ such that

\begin{align*}
&
R_n\paren{\max((\conv(\F^{(j)}_{\H}))_{j\in[k]})|\mathbf{x}\times \mathbf{y}}
\\
&\leq
C_\alpha\sqrt{k}\cdot \max_{j\in [k]}\max_{\mathbf{w}=w_{1:n}}R_n( \conv(\F_\H^{(j)})|\mathbf{w})\cdot \log^{\frac{3}{2}+\alpha}\Bigg(\frac{n}{\max_{j\in [k]}\max_{\mathbf{w}=w_{1:n}}R_n( \conv(\F_\H^{(j)})|\mathbf{w})}\Bigg) \\
&=
C_\alpha\sqrt{k}\cdot \max_{j\in [k]}\max_{\mathbf{w}=w_{1:n}}R_n( \F_\H^{(j)}|\mathbf{w})\cdot \log^{\frac{3}{2}+\alpha}\Bigg(\frac{n}{\max_{j\in [k]}\max_{\mathbf{w} = w_{1:n}}R_n(\F_\H^{(j)}|\mathbf{w})}\Bigg),
\end{align*}
where the last equality follows from
the well-known identity
$R_n(\F| \mathbf{w})=R_n(\conv(\F) | \mathbf{w})$, (see, e.g., \citet[Theorem 3.3]{boucheron05}). 

The function  
$x\mapsto x \log^{{3}/{2}+\alpha}
({n}/{x})$ has a maximum point at $x = n/e^{3/2+\alpha}$, and for $x\in (0,n/e^{3/2+\alpha}]$ is monotonic increasing. 
We bound the empirical Rademacher complexity (on any given sequence) via the $\VC$-dimension \citep{
bartlett2002rademacher}: $R_n(\F|\mathbf{w})\leq C\sqrt{\frac{\VC(\F)}{n}}$, \\
and for $\left(C\sqrt{\VC(\F_{\H})}e^{3/2+\alpha} \right)^{2/3}\leq n$, by the monotonicity of the function $x \log^{{3}/{2}+\alpha}
({n}/{x})$ we get an upper bound of
\begin{align*}
C_\alpha C\sqrt{\frac{k\max_{j\in[k]}\VC(\F_\H^{(j)})}{n}}\cdot& \log^{\frac{3}{2}+\alpha}\Bigg(\frac{n^{\frac{3}{2}}}{C\sqrt{\max_{j\in[k]}\VC(\F_\H^{(j)})}}\Bigg)
\\
&=
C_\alpha C\sqrt{\frac{k\VC(\H)}{n}}\cdot \log^{\frac{3}{2}+\alpha}\Bigg(\frac{n^{\frac{3}{2}}}{C\sqrt{\VC(\H)}}\Bigg) \\
&=
\O\Bigg(C_\alpha\sqrt{\frac{k\VC(\H)}{n}}\cdot \log^{\frac{3}{2}+\alpha}(n)\Bigg),    
\end{align*}
where the inequality follows from Lemma \ref{lem:vc-lossclass-graphdimension}.
We require that 
\beq
C_\alpha\sqrt{\frac{k\VC(\H)}{n}}\cdot \log^{\frac{3}{2}+\alpha}(n) + \sqrt{\frac{\log\left(\frac{1}{\delta}\right)}{n}} \leq \epsilon,
\eeq
and a standard inversion of this inequality yields sample complexity $n_0 = \O\big(\frac{C_\alpha}{\epsilon^2}(k\VC(\H)\log^{\frac{3}{2}+\alpha}(k\VC(\H))+\log(\frac{1}{\delta})\big)$.
\end{proof}
We find it instructive to provide an alternative (albeit worse) bound of
\begin{equation}\label{binary_valued_worse_bound}
R_n\paren{\max((\conv(\F^{(j)}_{\H}))_{j\in[k]})|\mathbf{x}\times \mathbf{y}}
\leq
\O\left(\sqrt{\frac{\VC(\H)\log^2(\VC(\H))k\log k\log^9(n)}{n}}\right)
\end{equation}
on
the Rademacher complexity, via 
a different technique (In Appendix \ref{app:gen-bound-binary-additional}).

\paragraph{Remark.} Theorem \ref{thm:binary-bound} provides an improvement to Theorem 7 in \cite*{JMLR:v19:17-402}, where they considered learning with intersection of hyperplanes for imbalanced binary classification problem.

\subsection{Multiclass Classification}\label{app:multiclass}
Let $\H\subseteq \Y^{\Z}$ be a function class such that
$\Y=[l]=\set{1\dots,l}$.
We follow similar arguments to the binary case.
\begin{theorem}[Generalization bound for multiclass classification]
\label{thm:multiclass-bound}
Let $\H$ be a function class with domain $\Z$ and range $\Y=[l]$ with finite Graph-dimension $d_G(\H)$.
For any $\alpha > 0$ there exists a constant $C_\alpha$ and there is a sample complexity
$n_0=\frac{C_\alpha}{\epsilon^2}\left(k\d_G(\H)\log^{\frac{3}{2}+\alpha}(k\d_G(\H))+\log(\frac{1}{\delta})\right)$,
such that for $|S|\geq n_0$, for every  $\Tilde{h}\in\Delta(\H)$,
\beq
|\Risk(\Tilde{h})-\mywidehat{\Risk}(\Tilde{h})|\leq \epsilon
\eeq
with probability at least $1-\delta$.
\end{theorem}
The following Lemma is standard and holds for the function classes $\F_\H^{(j)}$ (defined in \cref{def:jth-loss-class}).
\begin{lemma}
  \label{lem:vc-lossclass-graphdimension}
  Let $\H$ be a function class with domain $\Z$ and range $\Y=[l]$.
  Denote the Graph-dimension of $\H$ by $d_G(\H)$.
  Then for all $j \in [k]$
  \beq
  \VC(\F_\H^{(j)})\leq d_G(\H).
  \eeq
  In particular, for binary-valued classes, $\VC(\F_\H^{(j)})\leq \VC(\H)$ ---
  since for these, the VC- and Graph-dimensions coincide.
\end{lemma}
\bepf
Suppose that the binary function class
$\F_\H^{(j)}$
shatters the points $\set{(x_1,y_1),\ldots,(x_d,y_d)}\subset\X\times\Y$.
That means that for each $b\in\set{0,1}^d$, there is an $h_b\in\H$
such that
$\I\left[h_b(z_j(x_i))\neq y_i\right]=b_i$
for all $i\in[d]$,
where $z_j(x)$ is the $j$th element in the (ordered) set-valued output of $\rho$ on input $x$.
We claim that $\H$ is able to $G$-shatter
$S=\set{z_j(x_1),\ldots,z_j(x_d)}\subset\Z$.
Indeed, for each $T\subseteq S$, let $b=b(T)\in\set{0,1}^S$ be its
characteristic function. Taking $f:S\to\Y$ to be $f(x_i)=y_i$,
we see that the definition of $G$-shattering holds.
\enpf
For the proof of  Theorem \ref{thm:multiclass-bound}, we follow the
same proof of Theorem \ref{thm:binary-bound} and use the
Graph-dimension property of Lemma
\ref{lem:vc-lossclass-graphdimension}.
\paragraph{Remark.} A similar bound to that of Theorem
\ref{thm:binary-bound} can be achieved by using the Natarajan
dimension and the fact that \beq d_G(\H)\leq 4.67\log_2(|\Y|)d_N(\H)
\eeq as previously shown \citet{DBLP:journals/jcss/Ben-DavidCHL95}.

\section{Generalization Bounds For Regression}\label{sec:reg-bound}
Let $\H\subseteq\R^{\Z}$ be a hypothesis class of real functions. In the following, we provide three different generalization bounds, which,
as far as we can tell,
are mutually incomparable
uniformly over the parameter regimes.

\begin{theorem}[Generalization bound for Regression]\label{thm:regression-bound-1} Let $\H$ be a function class with domain $\Z$ and range $[0,1]$. Assume $\H$ has a finite $\gamma$-fat-shattering dimension for all $\gamma > 0$. 
Denote the sample size $|S|=n$ and
\begin{align*}
m_n(\H)=\inf_{\beta\geq 0}\set{4\beta +\O\left( \sqrt{\frac{\log^4(n)}{n}}\int_{\beta}^{1} \sqrt{\fat_{c\gamma}(\H) \log\left(\frac{1}{\gamma}\right)}d\gamma \right)},
\end{align*}
where c is a universal constant.
For the $L_1$ loss function and for every  $\Tilde{h}\in\Delta(\H)$, for any $\alpha>0$ there exist a constant $C_\alpha$ such that,
\begin{align*}
|\Risk(\Tilde{h})-\mywidehat{\Risk}(\Tilde{h})|\leq 
\O\left(C_\alpha \sqrt{k} \cdot m_n(\H)\cdot \log^{\frac{3}{2}+\alpha}\left(\frac{n}{m_n(\H)}
\right) + \sqrt{\frac{\log\left(\frac{1}{\delta}\right)}{n}}\right)
,    
\end{align*}
with probability at least $1-\delta$. 

Moreover, in the case of $L_2$ loss function, the same result holds with $\fat_\frac{c\gamma}{2}(\H)$ plugged into $m_n(\H)$.
\end{theorem}
In the following corollary (proof is in Appendix \ref{app:reg-hyperplane-bound-1}) we derive a simplified bound for hyperplanes. 
\begin{corollary}\label{regression-bound-hyperplanes-1}
Let $\H$ be a function class of homogeneous hyperplanes with domain $\R^m$. Using the same assumptions as in Theorem \ref{thm:regression-bound-1}, we have
\begin{align*}
|\Risk(\Tilde{h})-\mywidehat{\Risk}(\Tilde{h})|
\leq 
\O\left(C_\alpha   \sqrt{\frac{k}{n}}\log^{\frac{7}{2}}\left(n\right)\log^{\frac{3}{2}+\alpha}\left(\frac{n}{\log^{\frac{7}{2}}\left(n\right)}\right) + \sqrt{\frac{\log\left(\frac{1}{\delta}\right)}{n}}\right),
\end{align*}
with probability at least $1-\delta$. 
\end{corollary}
The class of hyperplanes can be learned with SGD, as the maximum of finite convex functions remains convex. However, our bound works for an arbitrary hypotheses class.

\begin{theorem}[Generalization bound for Regression]\label{thm:regression-bound-2} Let $\H$ be a function class with domain $\Z$ and range $[0,1]$. Assume $\H$ has a finite $\gamma$-fat-shattering dimension for all $\gamma > 0$. Denote the sample size $|S|=n$ and
\begin{eqnarray*}
  m_n(\H)=\inf_{\alpha \geq 0}
\set{4\alpha + \O\Bigg(\sqrt{\frac{k\log(k)\log^4(n)}{n}}\int_{\alpha}^{1} \sqrt{\log\left(\frac{1}{\gamma}\right)\Bigg( \frac{\fat_\frac{c\gamma}{4}(\H)}{\gamma^2}\log^2\bigg(\frac{\fat_\frac{c\gamma}{4}(\H)}{\gamma}\bigg)\Bigg)}d\gamma\Bigg)}.
\end{eqnarray*}
For the $L_1$ loss function and for every  $\Tilde{h}\in\Delta(\H)$,
\begin{align*}
|\Risk(\Tilde{h})-\mywidehat{\Risk}(\Tilde{h})|\leq \O\left(m_n(\H)+\sqrt{\frac{\log\left(\frac{1}{\delta}\right)}{n}}\right),    
\end{align*}
with probability at least $1-\delta$. 

Moreover, in the case of $L_2$ loss function, the same result holds with $\fat_\frac{c\gamma}{8}(\H)$ plugged into $m_n(\H)$.
\end{theorem}

\begin{theorem}[Generalization bound for Regression]\label{thm:regression-bound-3}
Let $\H$ be a function class with domain $\Z$ and range $[0,1]$. Assume $\H$ has a finite $\gamma$-fat-shattering dimension for all $\gamma > 0$. Denote the sample size $|S|=n$ and $d = \fat_\frac{\epsilon}{4}(\H)$.
For the $L_1$ loss function, there is a sample complexity 
\begin{align*}
n_0= \O 
     \left(
     \frac{1}{\epsilon^2}
     \left(
      k\log(k)\frac{d}{\epsilon^2}\log^2\frac{d}{\epsilon}\log^2\frac{1}{\epsilon}\log^2\left(k\log(k)\frac{d}{\epsilon^4}\log^2\frac{d}{\epsilon}\log^2\frac{1}{\epsilon}\right)
      +
     \log\frac{1}{\delta}
     \right)
     \right),
     \end{align*}
     such that for $|S|\geq n_0$, for every  $\Tilde{h}\in\Delta(\H)$    
\beq
|\Risk(\Tilde{h})-\mywidehat{\Risk}(\Tilde{h})|\leq \epsilon
\eeq
with probability at least $1-\delta$. 
\end{theorem}

We would like to compare the bounds in \cref{thm:regression-bound-1,thm:regression-bound-2,thm:regression-bound-3}.
In terms of dependence in the fat-shattering dimension and $k$, \cref{thm:regression-bound-1} would give a better bound than \cref{thm:regression-bound-2}. However, the latter has a better dependence in $\log(n)$ factors. Regarding \cref{thm:regression-bound-3}, on the one hand, the dependence in $n$ (sample size) is $1/n^{1/4}$. On the other hand, we have the fat-shattering dimension with a specific scale (the error parameter, $\epsilon$). In some cases, we can obtain an improved learning rate. For example, by taking $\fat_\gamma(\H) = 1/\gamma^6$, \cref{thm:regression-bound-1} guarantees learning rate of $1/n^{1/6}$ and so \cref{thm:regression-bound-3} provides a sharper bound.

\subsection{Shattering Dimension of the Class $\max\paren{(\A^{(j)})_{j\in[k]})}$}
The main result of this section is bounding the fat shattering dimension of $\max\paren{(\A^{(j)})_{j\in[k]}}$ class.
\begin{theorem}[Fat-shattering of $k$-fold maxima]
\label{thm:fat-kmax}
Let $S=\set{x_1,\ldots,x_m}$.
For any $k$ real valued functions classes $\F_{1},\dots,\F_{k} \subseteq \R^S$,
\beq
\fat_\gamma\paren{\max\paren{(\F_{j})_{j\in[k]}}}
\leq
\O\paren{\log(k)\log^2(m)\sum_{j=1}^k\fat_\gamma(\F_{j})}.
\eeq
\end{theorem}
\paragraph{Remark.} It was pointed out to us by Yann Guermeur that the 
corresponding result in the conference version of this paper, 
\citet{attias2019improved}, Theorem 12, contained a mistake
--- the root of which was an erroneous claim in Lemma 14 therein. The corrected version of that result was proved by \cite*{alon2022theory},
Lemma~\ref{lem:disambig} below, allows for a corrected version of Theorem 12, with an additional $\log^2(m)$ factor. 
In a subsequent work \citep*[Theorem 1]{kontorovich2021fat}, this result was further improved,
\beq
\fat_\gamma\paren{\max\paren{(\F_{j})_{j\in[k]}}}
\leq
\O\paren{\sum_{j=1}^k\fat_\gamma(\F_{j})\log^2\paren{\sum_{j=1}^k\fat_\gamma(\F_{j})}}.
\eeq
Moreover, \cite{attias2022adversarially} studied the regression setting with $\ell_p$ losses and
arbitrary perturbation sets.

Before presenting the proof, we introduce some auxiliary notions.
We say that $\F$ ``$\gamma$-shatters
a set $S$ at zero'' if the shift 
$r$ is constrained to be $0$
in the the usual $\gamma$-shattering definition
(has appeared previously in
\citet{gottlieb2014efficient}).
The analogous dimension will be denoted by $\fat_\gamma^0(\F)$.

\begin{lemma}
\label{lem:fat0}
For all $\F\subseteq\R^\X$ and $\gamma>0$, we have
\begin{align}
\label{eq:fat-aux}
 \fat_\gamma(\F) = \max_{r\in\R^\X} \fat_\gamma^0(\F-r),
\end{align}
where $\F-r=\set{f-r: f\in\F}$ is the $r$-shifted class; in particular,
the maximum is always achieved.
\end{lemma}
\bepf
Fix $\F$ and $\gamma$. For any choice of $r\in\R^\X$,
if $\F-r$
$\gamma$-shatters some set $S\subseteq\X$ at zero,
then then $\F$ $\gamma$-shatters $S$ in the usual sense
with shift $r_S\in\R^S$ (i.e., the restriction of $r$ to $S$).
This proves that the left-hand side of \cref{eq:fat-aux}
is at least as large as the right-hand side.
Conversely, suppose that
$\F$ $\gamma$-shatters some $S\subseteq\X$ in the usual sense, with some shift $r\in\R^S$.
Choosing $r'\in\R^\X$ by $r'_S=r$ and $r'_{\X\setminus S}=0$,
we see that $\F-r'$ $\gamma$-shatters
$S$ at zero. This proves the other direction and hence the claim.
\enpf
Consider an {\em ambiguous} function class
$F^\star\subseteq\set{0,1,\star}^X$.
We say that
$F^\star$ {\em shatters} a set $S\subseteq X$
if $F^\star(S)\supseteq\set{0,1}^S$.
We say that $\bar f\in\set{0,1}^X$ is a {\em disambiguation}
of $f^\star\in F^\star$ if
the two functions agree on $x\in X$
whenever $f^\star(x)\neq\star$.
We say that $\bar F\subseteq\set{0,1}^X$ is a disambiguation of
$F^\star$
if each $\bar f\in \bar F$
is a disambiguation of
{\em some} $f^\star\in F^\star$
and {\em every} $f^\star\in F^\star$
has a disambiguated representative $\bar f\in \bar F$.
We define $\VC(F^\star)$ as the maximum size
of a shattered set (possibly, $\infty$).

It will be convenient to visually represent
such function classes as (possibly infinite)
matrices, where the rows correspond to $f\in F$
and the columns correspond to $x\in X$.

\begin{example}
\label{ex:1}
  It might be the case that
  $\VC(F^\star)=1$
  while
  any disambiguation 
  $\bar F$
  verifies
  $\VC(\bar F)=2$:
   $$\bordermatrix{ &x_1 &x_2& x_3\cr
            f_1&    1 & 1 & 1\cr
            f_2&    0 & 1 & 1 \cr
            f_3&    1 & 0 & 1 \cr
            f_4&    \star & 0 & 0 \cr
            f_5&    0 & \star & 0\cr }.$$
It was mistakenly claimed in the conference version
\citep[Lemma 14]{attias2019improved}
that one can always find a disambiguation
$\bar F$ such that
$\VC(\bar F)\le \VC(F^\star)$.
We thank Yann Guermeur for pointing out
this error.
\end{example}

The following result provides a generic disambiguation rule that upper bounds the size of any disambiguated function classes. We reproduce it in \cref{app:gen-bound-binary-additional} for completeness.

\begin{lemma}\cite*[Theorem 13]{alon2022theory}
\label{lem:disambig}\label{lem:disambig-rule}
For $X=\N=\set{1,2,\ldots}$ and any $F^\star\subseteq\set{0,1,\star}^X$
with $\VC(F^\star)\le d$,
there is a disambiguation
$\bar F\subseteq\set{0,1}^X$ with the following property:
For each prefix $X_m:=[m]=\set{1,2,\ldots,m}$, we have
\beq
|\bar F(X_m)| \le m^{\O\paren{d\log m}}
.
\eeq
\end{lemma}

\begin{example}\citep*{alon2022theory}\label{ex:2}
Consider the following ambiguous class $F^\star$ consisting
of $5$ functions acting on the $3$ points $X=\set{x_1,x_2,x_3}$:
\beq
\bordermatrix{ &x_1 &x_2& x_3\cr
            f_1&    0 & 0 & 0 \cr
            f_2&    1 & 1 & 1 \cr
            f_3&    \star & 1 & 0 \cr
            f_4&    0 & \star & 1 \cr
            f_5&    1 & 0 & \star \cr }.
\eeq
It is straightforward to verify that $\VC(F^*)=1$
and further that any disambiguation $\bar F$
verifies $|\bar F(X)|=5$.
Contrast this with the Sauer-Shelah lemma, which
upper-bounds the number of behaviors that a class of VC-dimension
$1$ can achieve on $3$ points by $4$.
\end{example}

\paragraph{Remark.} There exist an ambiguous function class $F^\star$, such that for any disambiguation $\bar F$ it holds that $\VC(\bar F) = \infty$.
See \cite*{alon2022theory}, Theorem 1.

\hide{
\begin{example}\label{ex:2}(due to Steve Hanneke, Ron Holzman, Shay Moran)
The ambiguous class can only shatter one point. Any disambiguation resolving $\star$ necessarily entails shattering two points. 
Resolving $f_3$ with $0$ involves shattering $J=\set{1,2}$ 
and resolving it with $1$ involves shattering $J=\set{1,3}$.
Moreover, the size of distinct rows is $\Phi(3,1)+1$ for any disambiguation.
   $$\bordermatrix{ &x_1 &x_2& x_3\cr
            f_1&    0 & 0 & 0 \cr
            f_2&    1 & 1 & 1 \cr
            f_3&    \star & 1 & 0 \cr
            f_4&    0 & \star & 1 \cr
            f_5&    1 & 0 & \star \cr }$$
\end{example}
\begin{example}\label{ex:3}
The ambiguous class can only shatter one point. Any disambiguation resolving $\star$ necessarily entails shattering two points, for example, resolving $f_3$ with $1$ involves shattering $J=\set{4,5}$ 
and resolving it with $0$ involves shattering $J=\set{1,5}$.
Moreover, the size of distinct rows is $\Phi(5,1)+2$ for any disambiguation.
$$\bordermatrix{ &x_1 &x_2& x_3 & x_4 & x_5         \cr
            f_1& 1      &   0       &   1     &  0      &   1  \cr
            f_2& 1      &   1       &   1     &  \star  &   1  \cr
            f_3& 1      &   1       &   0     &   1     &   \star  \cr
            f_4& 1      &   \star   &   0     &   0     &   1  \cr
            f_5& 0      &   1       &   1     &   \star &   1  \cr
            f_6& 0      &   1       &   \star &   1     &   0  \cr
            f_7& 1      &   0       &   \star &   1     &   \star  \cr
            f_8& \star  &   \star   &   0     &   0     &   0  \cr  }$$

\end{example}

\begin{example}\label{ex:4}(due to Uri Stemmer)
The ambiguous class can only shatter two points. Any disambiguation resolving $\star$ necessarily entails shattering three points.
Resolving $\star$ with $0$ involves shattering $J=\set{3,4,5}$
and resolving it with $1$ involves shattering $J=\set{1,2,5}$.

$$\bordermatrix{ &x_1 &x_2& x_3 & x_4 & x_5         \cr
            f_1&     0  &   0  &   0  &  0   &   0  \cr
            f_2&     0  &   0  &   1  &  0   &   0  \cr
            f_3&     0  &   0  &   1  &   1  &   0  \cr
            f_4&     0  &   0  &   0  &   0  &   1  \cr
            f_5&     0  &   0  &   0  &   1  &   1  \cr
            f_6&     0  &   0  &   1  &   0  &   1  \cr
            f_7&     0  &   0  &   1  &   1  &   1  \cr
            f_8&     0  &   1  &   0  &   0  &   0  \cr  
            f_9&     0  &   1  &   1  &   0  &   0  \cr
            f_{10}&  1  &   0  &   0  &   0  &   0  \cr
            f_{11}&  1  &   0  &   1  &   0  &   0  \cr
            f_{12}&  1  &   0  &   0  &   0  &   1  \cr
            f_{13}&  1  &   0  &   0  &   1  &   1  \cr 
            f_{14}&  1  &   1  &   0  &   0  &   0  \cr
            f_{15}&  1  &   1  &   0  &   0  &   1  \cr
            f_{16}&  0  &   1  &   0  &   1  &   \star  \cr }$$
                
\end{example}  

}
\begin{lemma}
  \label{lem:k-fold-op}
Let $G:\set{-1,1}^k\rightarrow \set{-1,1}$
and let $\F_1,\ldots,\F_k\subseteq\set{-1,1}^\X$ be hypothesis classes with $\VC(\F_j)=d_j$. Denote $\bar d:=\frac1k\sum_{i=1}^k d_j$. Define the function class $G\paren{\F_1,\ldots,\F_k}=:
\set{\X\ni x \mapsto G\paren{f_1(x),\ldots,f_k(x)}: f_i \in \F_i}$.
Then,
$$\VC\paren{G\paren{\F_1,\ldots,\F_k}}
\leq
2k\log(3k)\bar d
$$
\end{lemma}
\bepf
We adapt the argument of
\citet[Lemma 3.2.3]{MR1072253}, which is stated therein for $k$-fold unions
and intersections.
The $k=1$ case is trivial, so assume $k\ge 2$.
For any $S\subseteq\X$,
define $G\paren{\F_1,\ldots,\F_k}(S)\subseteq\set{-1,1}^S$
to be the restriction of $G\paren{\F_1,\ldots,\F_k}$ to $S$.
The key observation is that
\beq
|G\paren{\F_1,\ldots,\F_k}
(S)|
&\le& \prod_{j=1}^k |\F_j(S)|\\
&\le&
\prod_{j=1}^k(e|S|/d_j)^{d_j} \\
&\le&
(e|S|/\bar d)^{\bar d k}.
\eeq
The last inequality requires proof.
After taking logarithms and dividing both sides by $k$,
it is equivalent to the claim that
$$ \bar d\log\bar d\le \frac1k\sum_{j=1}^k d_j\log d_j\;,$$
an immediate consequence of Jensen's inequality applied to the convex function $f(x)=x\log x$.

The rest of the argument is identical that of \citet{MR1072253}:
one readily verifies that
for $m=|S|=2\bar d k\log(3k)$,
we have
$(em/\bar d)^{\bar d k}<2^m$.
\enpf

\bepf [Proof of Theorem \ref{thm:fat-kmax}]
To prove the Theorem,
it suffices to show that for all $\F_j\subseteq\R^S$
\begin{align}
\label{eq:kfat0}
\fat^0_\gamma(\max((\F_{j})_{j\in[k]}))
\le
\O(\log(k)\log^2(m)\sum_{j=1}^k\fat^0_\gamma(\F_{j})).
\end{align}
Indeed,
we observe that $r$-shift commutes with the max operator:
\begin{align}\label{eq:r-shift-max-commutes}
\max((\F_j-r)_{j\in[k]})
= \max((\F_{j})_{j\in[k]})-r. 
\end{align}
By applying
Lemma~\ref{lem:fat0} to the function class $\max((\F_{j})_{j\in[k]})$ and using \cref{eq:r-shift-max-commutes}, we have
\beq
\fat_\gamma(
\max((\F_{j})_{j\in[k]})
) =
\max_r \fat_\gamma^0(
\max((\F_{j})_{j\in[k]})
-r
)
=\max_r \fat_\gamma^0(
\max((\F_j-r)_{j\in[k]})
)
.
\eeq
Applying \cref{eq:kfat0} to classes $F_j-r$ obtains
\begin{align*}
    \max_r \fat_\gamma^0(
    \max((\F_j-r)_{j\in[k]})
    \leq
    \max_r \O(\log(k)\log^2(m)\sum_{j=1}^k\fat^0_\gamma(\F_j-r)),
\end{align*}
Then,
\begin{align*}
    \max_r \O(\log(k)\log^2(m)\sum_{j=1}^k\fat^0_\gamma(\F_j-r))
    &
    \leq
    \O(\log(k)\log^2(m)\sum_{j=1}^k\max_{r_j}\fat^0_\gamma(\F_j-r_j))
    \\
    &
    =
    \O(\log(k)\log^2(m)\sum_{j=1}^k\fat_\gamma(\F_j)),
\end{align*}
where the last identity follows from Lemma~\ref{lem:fat0}.

Now we proceed to prove \cref{eq:kfat0}. 
First, convert $\F_j\subseteq\R^S$ to a finite class
$\F_j^\star\subseteq\set{-\gamma,\gamma,\star}^S$ for $S=\set{x_1,\ldots,x_m}$, as follows.
For every vector in $v\in\F_j$,
define $v^\star\in\F_j^\star$ by:
$v^\star_i=\sgn(v_i)\gamma$ if $|v_i|\ge\gamma$
and
$v^\star_i=\star$ else.
The notion of shattering (at zero) remains the same:
a set $T\subseteq S$ is shattered if
$\set{-\gamma,\gamma}^T \subseteq \F_j^\star(T)$. Note that $\F_j^\star$ and $\F_j$ has the same $\gamma$-shattering dimension at zero.

Lemma~\ref{lem:disambig}
furnishes
a mapping $\phi:\F_j^\star\to\set{-\gamma,\gamma}^S$
such that
(i) for all $v\in\F_j^\star$ and all $i\in[m]$, we have $v_i\neq\star\implies (\phi(v))_i=v_i$
and (ii) $\phi(\F_j^\star)$ does not shatter more points than $\F_j^\star$ times $\log^2(m)$.
Together, properties (i) and (ii) imply that for all $j\in [k],$
\beq
\fat^0_\gamma(\phi(\F_j^\star))
\leq 
\O(\fat^0_\gamma(\F_j)\cdot\log^2(m)).
\eeq

Finally, observe that any set of points in $S$ $\gamma$-shattered by
$\max((\F_j)_{j\in[k]})$ are also shattered by\\ $\max((\phi(\F_j^\star))_{j\in[k]})$.
Applying Lemma~\ref{lem:k-fold-op}
with $G(f_1,\ldots,f_k)(x)=\max_{j\in[k]}f_j(x)$
shows that\\
$\max((\phi(\F_j^\star))_{j\in[k]})$
cannot shatter
$2\log(3k)\sum_{j=1}^kd_j$
points, where
\beq
d_j=
\fat^0_\gamma(\phi(\F_j^\star))
\leq 
\O(\fat^0_\gamma(\F_j)\cdot\log^2(m))
.
\eeq
We have shown that,
\beq
\fat^0_\gamma(\max((\F_j)_{j\in[k]}))
\le
\O(\log(k)\log^2(m)\sum_{j=1}^k\fat^0_\gamma(\F_j)),
\eeq
this concludes the proof of \cref{eq:kfat0}.
\enpf
\subsection{Shattering Dimension of $L_1$ and $L_2$ Loss Classes}
\begin{lemma}
  \label{lem:fatdimension-lossclassL2-bound}Let $\H \subset\R^m$ be a real valued function class on m points. denote $L_\H^1$ and $L_\H^2$ the $L_1$ and $L_2$ loss classes of $\H$ respectively. Assume $L_\H^2$ is bounded by $M$. For any $\H$,
\beq
\fat_\gamma(L_\H^1)&\le& \O(\log^2(m)\fat_{\gamma}(\H)),\;\;\;\mbox{ and }\;\;\;
\fat_\gamma(L_\H^2)\le \O(\log^2(m)\fat_{\gamma/2M}(\H)).
\eeq
\end{lemma}

\begin{lemma}\label{lem:fat-losses}
Let $\ell: \Y \times \Y \rightarrow \R $ be an arbitrary loss function. For $j\in[k]$ define 
\beq
    \F_\H^{(j),\ell}:=\set{\X\times \Y \ni (x,y)\mapsto \ell(h(z_j),y):h\in \H,\;\rho(x)=\set{z_1,\dots,z_k}},
\eeq
and 
\beq
    L_\H^\ell :=  \set{Z\times\Y\ni(z,y)\mapsto \ell(h(z),y): h\in \H}.
\eeq
Then, for all $\gamma>0$,
\beq
\fat_\gamma(\F^{(j),\ell}_{\H}) \le \fat_\gamma(L_\H^\ell).
\eeq
\end{lemma}

\bepf
The claim stems from the inclusion $\F^{(j),\ell}_{\H} \subseteq L_\H^\ell$.
\enpf

\hide{
\begin{lemma}\label{vc-unionclass-bound}
For $\H,\G\subseteq\set{-1,1}^X$ with finite VC-dimension, we have
\beq
\VC(\H\cup \G)\leq \VC(\H)+\VC(\G)+1,
\eeq
and for
$\H,\G\subseteq \R^X$ with finite $\gamma$-fat-shattering dimension, we have
\beq
\fat_{\gamma}(\H\cup \G)\leq \fat_{\gamma}(\H)+\fat_{\gamma}(\G)+1.
\eeq
\end{lemma}
}

\begin{proof}[Proof of Lemma~\ref{lem:fatdimension-lossclassL2-bound}]
  For any $\X$ and any function class $\H \subset\R^\X$,
define the {\em difference class} $\H^\Delta\subset\R^{\X\times\R}$ as
\beq
\H^\Delta = \set{ \X\times\R\ni (x,y)\mapsto \Delta_h(x,y):= h(x)-y ; h\in\H}.
\eeq
In words:
$\H^\Delta$
consists of all functions $\Delta_h(x,y)= h(x)-y$ indexed by $h\in\H$.

It is easy to see that for all $\gamma>0$, we have
$\fat_\gamma(\H^\Delta)\le\fat_\gamma(\H)$.
Indeed, if $\H^\Delta$ $\gamma$-shatters some set
$\set{(x_1,y_1),\ldots,(x_k,y_k)}\subset\X\times\R$
with shift $r\in\R^k$,
then $\H$ $\gamma$-shatters the set
$\set{x_1,\ldots,x_k}\subset\X$
with shift $r+(y_1,\ldots,y_k)$.

Next, we observe that taking the absolute value does not significantly increase
the fat-shattering dimension. Indeed, for any real-valued function class $\F$,
define
$\absop(\F)
:=\set{|f|;\;f\in\F}$.
Observe that $\absop(\F)\subseteq \max((F_j)_{j\in[2]})$,
where $\F_1=\F$ and $\F_2=-\F=:\set{-f;f\in\F}$.
It follows from Theorem~\ref{thm:fat-kmax} that
\begin{align}
\label{eq:fat-abs}
\fat_\gamma(\absop(\F))<\O(\log^2(m)(\fat_\gamma(\F)+\fat_\gamma(-\F)))< \O(\log^2(m)\fat_\gamma(\F)).
\end{align}
\hide{
Clearly,
$\fat_\gamma(\absop(\F))=\fat_\gamma(\F)$.
Because $\F$ is positive. Then,
\beq
\fat_\gamma(\absop(\F)) &\le& \fat_\gamma(\absop(\F_+)))+\fat_\gamma(\absop(\F_-))+1 \\
&=&
2\fat_{\gamma}(\F)+1,
\eeq
Finally,
}
Next,
define $\F$ as the $L_1$ loss class of $\H$:
\beq
\F = \set{ \X\times\R\ni (x,y)\mapsto |h(x)-y)| ; h\in\H}.
\eeq
Then
\beq
\fat_\gamma(\F) &=& \fat_\gamma(\absop(\H^\Delta)) \\
&\le& \O(\log^2(m)\fat_{\gamma}(\H^\Delta)) \\
&\le& \O(\log^2(m)\fat_{\gamma}(\H));
\eeq
this proves the claim for $L_1$.

To analyze the $L_2$ case,
consider
$\F\subset[0,M]^\X$ and define
$\F^{\circ2}
:=\set{f^2;f\in\F}$.
We would like to bound
$\fat_\gamma(\F^{\circ2})$
in terms of
$\fat_\gamma(\F)$.
Suppose that $\F^{\circ2}$ $\gamma$-shatters
some set
$\set{x_1,\ldots,x_k}$
with shift $r^2=(r_1^2,\ldots,r_k^2)\in[0,M]^k$
(there is no loss of generality in assuming
that the shift has the same range as the function class).
Using the elementary inequality
\beq
|a^2-b^2| \le 2M|a-b|,
\qquad
a,b\in[0,M],
\eeq
we conclude that $\F$ is able to $\gamma/(2M)$-shatter
the same $k$ points
and thus
$\fat_\gamma(\F^{\circ2})
\le
\fat_{\gamma/(2M)}(\F)$.

To extend this result to the case
where $\F\subset[-M,M]^\X$, we use \cref{eq:fat-abs}.
\hide{
decompose $\F$
as $\F=\F_+\cup\F_-$, where $\F_+
=\F\cap[0,M]^\X
$
and
$\F_-
=\F\cap[-M,0]^\X
$.
Clearly, then,
$
\F^{\circ2}
=
(\F_+)^{\circ2}
\cup
(\F_-)^{\circ2}
$.
Furthermore, using Lemma \ref{vc-unionclass-bound} we observe that for any $\H,\G\subset\R^\X$, we have
\beq
\fat_\gamma(\H\cup\G)\le \fat_\gamma(\H)+\fat_\gamma(\G)+1,
\eeq
whence
\beq
\fat_\gamma(\F^{\circ2}) &\le& \fat_\gamma((\F_+)^{\circ2})+\fat_\gamma((\F_-)^{\circ2})+1 \\
&\le&
\fat_{\gamma/(2M)}(\F_+)
+
\fat_{\gamma/(2M)}(\F_-)
+1 \\
&\le&
2\fat_{\gamma/(2M)}(\F)+1,
\eeq
where the second inequality follows from the fact that $\fat_\gamma(-\F)=\fat_\gamma(\F)$.
}
In particular,
define $\F$ as the $L_2$ loss class of $\H$:
\beq
\F = \set{ \X\times\R\ni (x,y)\mapsto (h(x)-y)^2 ; h\in\H}.
\eeq

Then
\beq
\fat_\gamma(\F) &=& \fat_\gamma((\H^\Delta)^{\circ2}) \\
&=& \fat_\gamma((\absop(\H^\Delta))^{\circ2}) \\
&\le& \fat_{\gamma/(2M)}(\absop(\H^\Delta)) \\
&\le& \O(\log^2(m)\fat_{\gamma/(2M)}(\H^\Delta)) \\
&\le& \O(\log^2(m)\fat_{\gamma/(2M)}(\H)).
\eeq

\end{proof}
\subsection{Auxiliary Results}
Finally, before providing formal proofs, we use the following result on the fat-shattering of convex hulls. We then conclude a bound on the fat-shattering dimension of $k$-fold maximum of convex hulls using Theorem \ref{thm:fat-kmax}.
\begin{theorem}\cite[Theorem 1.5]{mendelson2001size}\label{thm:fat-conv} There is an absolute constant $C$, such that for every function class $F$ bounded by $[0,1]$ and every $\gamma >0$,
\beq
\fat_{\gamma}(\conv(F))
\leq
C\frac{\fat_\frac{\gamma}{4}(\F)}{\gamma^2}\log^2\left(\frac{2\fat_\frac{\gamma}{4}(\F)}{\gamma}\right)
\eeq
\end{theorem}

\begin{corollary}\label{cor:fat-max-conv-bound}
Let $S=\set{x_1,\ldots,x_m}$. For any $k$ real valued functions classes $\F_{1},\dots,\F_{k} \subseteq [0,1]^S$,
\beq
\fat_\gamma(\max((\conv(\F_{j}))_{j\in[k]}))
\leq
\O \left( k\log(k)\log^2(m) \max_{j\in[k]}\left(\frac{\fat_\frac{\gamma}{4}(\F_{j})}{\gamma^2}\log^2\left(\frac{\fat_\frac{\gamma}{4}(\F_{j})}{\gamma}\right)\right)\right).
\eeq

\end{corollary}

\begin{proof}
\beq
\fat_\gamma(\max((\conv(\F_{j}))_{j\in[k]})(S))
&\overset{(i)}
\leq &
\O\paren{\log(k)\log^2(m)\sum_{j=1}^k\fat_\gamma(\conv(\F_{j}))} \\
&\overset{(ii)}\leq&
\O\paren{\log(k)\log^2(m)\sum_{j=1}^k \frac{\fat_\frac{\gamma}{4}(\F_{j})}{\gamma^2}\log^2\left(\frac{\fat_\frac{\gamma}{4}(\F_{j})}{\gamma}\right)}\\
&\leq&
\O\paren{k\log(k)\log^2(m) \max_{j\in[k]}\Bigg(\frac{\fat_\frac{\gamma}{4}(\F_{j})}{\gamma^2}\log^2\left(\frac{\fat_\frac{\gamma}{4}(\F_{j})}{\gamma}\right)\Bigg)},
\eeq
where (i) stems from Theorem \ref{thm:fat-kmax} and (ii) stems from Theorem \ref{thm:fat-conv}.
\end{proof}

\begin{theorem}\label{thm:rad-fatdimension-bound} \citep{dudley1967sizes,MR1965359} For any $\F\subseteq[-1,1]^X$, any $\gamma \in (0,1)$ and $S=(w_1,\ldots,w_n)=\mathbf{w}\in \W^n$,
\beq
R_n(\F|\mathbf{w})\leq \sqrt{\frac{C}{n}}\int_{0}^{1} \sqrt{\fat_{c\gamma}(\F) \log\left(\frac{2}{\gamma}\right)}d\gamma,
\eeq
where  $c$ and $C$ are universal constants.

When the integral above diverges, the bound can be refined by
\beq
R_n(\F|\mathbf{w})\leq  
\inf_{\alpha \geq 0}
\set{4\alpha + \sqrt{\frac{C}{n}}\int_{\alpha}^{1} \sqrt{\fat_{c\gamma}(\F) \log\left(\frac{2}{\gamma}\right)}d\gamma}.
\eeq
\end{theorem}
\subsection{Proofs}
We now formally prove our main results for this section, generalization bounds in the case of real-valued functions.
\begin{proof}[Proof of Theorem~\ref{thm:regression-bound-1}]
We follow the same steps as in the proof of \cref{thm:binary-bound}
with two changes. The first one is bounding the empirical Rademacher complexity via the fat-shattering dimension (instead of the $\VC$-dimension in the binary case), using \cref{thm:rad-fatdimension-bound},
\begin{align*}
 R_n(\F|\mathbf{w})\leq  
\inf_{\beta \geq 0}
\set{4\beta + \sqrt{\frac{C}{n}}\int_{\beta}^{1} \sqrt{\fat_{c\gamma}(\F) \log\left(\frac{2}{\gamma}\right)}d\gamma}:=g_n(\F),
\end{align*}
this bound holds for every sequence of points.
The second difference is that we now need to bound the maximum fat-shattering dimension (instead of the $\VC$-dimension) over the classes $\F_{\H}^{(j)}$, for that purpose we use Lemma \ref{lem:fatdimension-lossclassL2-bound} and Lemma\ref{lem:fat-losses},
\begin{align*}
    \max_{j\in[k]}\fat_{\gamma}(\F_{\H}^{(j)})
    \leq
    \O(\log^2(n)\fat_{\gamma}(\H)).
\end{align*}
Denote
\begin{align*}
m_n(\H)=\inf_{\beta\geq 0}\set{4\beta +\O\left( \sqrt{\frac{\log^4(n)}{n}}\int_{\beta}^{1} \sqrt{\fat_{c\gamma}(\H) \log\left(\frac{1}{\gamma}\right)}d\gamma \right)}.
\end{align*}
Similar to \cref{thm:binary-bound}, the function $x \log^{{3}/{2}+\alpha} ({n}/{x})$ is monotonic increasing for $x\in (0,n/e^{3/2+\alpha}]$. For sufficiently large $n$ $\left(g_n(\F)\leq n/e^{3/2+\alpha}\right)$  and considering the aforementioned changes 
we have that for any $\alpha >0$ there exists a constant $C_\alpha>0$ such that

\begin{align*}
&
R_n(\max((\conv(\F^{(j)}_{\H}))_{j\in[k]})|\mathbf{x}\times \mathbf{y})
\\
&\leq
C_\alpha\sqrt{k}\cdot \max_{j\in [k]}\max_{\mathbf{w}=w_{1:n}}R_n( \F_\H^{(j)}|\mathbf{w})\cdot \log^{\frac{3}{2}+\alpha}\Bigg(\frac{n}{\max_{j\in [k]}\max_{\mathbf{w} = w_{1:n}}R_n(\F_\H^{(j)}|\mathbf{w})}\Bigg)
\\
&\leq
\O\left( C_\alpha \sqrt{k}\max_{j\in[k]}g_n(\F_\H^{(j)})
\cdot \log^{\frac{3}{2}+\alpha}\left(\frac{n}{\max_{j\in[k]}g_n(\F_\H^{(j)})}\right) \right)
\\
&=
\O\left( C_\alpha \sqrt{k} \cdot m_n(\H)\cdot \log^{\frac{3}{2}+\alpha}\left(\frac{n}{m_n(\H)}
\right)\right).
\end{align*}
We conclude that
\beq
|\Risk(\Tilde{h})-\mywidehat{\Risk}(\Tilde{h})|
\leq
\O\left(C_\alpha \sqrt{k} \cdot m_n(\H)\cdot \log^{\frac{3}{2}+\alpha}\left(\frac{n}{m_n(\H)}
\right) + \sqrt{\frac{\log\left(\frac{1}{\delta}\right)}{n}}\right).
\eeq

\end{proof}
\begin{proof}[Proof of Theorem~\ref{thm:regression-bound-2}]
  Similar to the proof for binary case,
  we bound the empirical Rademacher complexity of the loss class of $\Tilde{h}\in\Delta(\H)$.
\beq
|\Risk(\Tilde{h})-\mywidehat{\Risk}(\Tilde{h})|&=& |E_{(x,y)\sim D}\max_{j\in[k]}\sum_{t=1}^T \alpha_t f_t^{(j)}(x,y)-\frac{1}{n}\underset{(x,y)\in S}{{\displaystyle \sum}}\max_{j\in[k]}\sum_{t=1}^T \alpha_t f_t^{(j)}(x,y)| \\
&\le&
2R_n(\max((\conv(\F^{(j)}_{\H}))_{j\in[k]})|\mathbf{x}\times \mathbf{y})+3\sqrt{\frac{\log\left(\frac{2}{\delta}\right)}{2n}},
\eeq
where the inequality stems from applying Theorem \ref{thm:rad-gen-bound} 
on the function class $\conv^{\rho}(L_\H)$ and \cref{eq:conv-rho-max}.
From Theorem \ref{thm:rad-fatdimension-bound} we have
\begin{align*}
&
R_n(\max((\conv(\F^{(j)}_{\H}))_{j\in[k]})|\mathbf{x}\times \mathbf{y})
\\
&\leq
\inf_{\alpha \geq 0}
\set{4\alpha + \sqrt{\frac{C_1}{n}}\int_{\alpha}^{1} \sqrt{\fat_{c\gamma}(\max((\conv(\F^{(j)}_{\H}))_{j\in[k]})) \log\left(\frac{2}{\gamma}\right)}d\gamma}.
\end{align*}
Using Corollary \ref{cor:fat-max-conv-bound} we upper bound the inner term by
\beq
\O\paren{\sqrt{\frac{k\log(k)\log^2(n)}{n}}\int_{\alpha}^{1} \sqrt{\log\left(\frac{1}{\gamma}\right)\max_{j\in[k]}\Bigg( \frac{\fat_\frac{c\gamma}{4}(\F^{(j)}_{\H}(S))}{\gamma^2}\log^2\bigg(\frac{\fat_\frac{c\gamma}{4}(\F^{(j)}_{\H}(S))}{\gamma}\bigg)\Bigg)}d\gamma}.
\eeq 
Lemmas \ref{lem:fatdimension-lossclassL2-bound}
and \ref{lem:fat-losses} concludes the proof with
\beq
\O\Bigg(\sqrt{\frac{k\log(k)\log^4(n)}{n}}\int_{\alpha}^{1} \sqrt{\log\left(\frac{1}{\gamma}\right)\Bigg( \frac{\fat_\frac{c\gamma}{4}(\H)}{\gamma^2}\log^2\bigg(\frac{\fat_\frac{c\gamma}{4}(\H)}{\gamma}\bigg)\Bigg)}d\gamma\Bigg).
\eeq

\end{proof}

\begin{proof}[Proof of Theorem~\ref{thm:regression-bound-3}]
Denote the sample size by $|S|=n$.
We start off with a known generalization bound by \citet{bartlett1998prediction}, showing that for any function class $\H : \Z \rightarrow [0,1]$, the sample size is at least
\begin{align*}
n \leq \O\left(\frac{1}{\epsilon^2}\left(\fat_{\frac{\epsilon}{5}}(\H)\log^2\frac{1}{\epsilon}+\log\frac{1}{\delta}
\right)\right). 
\end{align*}
In our case, the function class
we are interested in is $\max((\conv(\F^{(j)}_{\H}))_{j\in[k]})$.
by Corollary \ref{cor:fat-max-conv-bound} we have that
\begin{align*}
\fat_{\epsilon}(\max((\conv(\F^{(j)}_{\H}))_{j\in[k]}))
\leq
\O\left(
k\log(k)\log^2(n)\left(\frac{\fat_\frac{\epsilon}{4}(\H)}{\epsilon^2}\log^2\left(\frac{\fat_\frac{\epsilon}{4}(\H)}{\epsilon}\right)\right)\right).   
\end{align*}
Thus, it suffices to solve the following
\begin{align*}
    n\leq \O\left(\left(\frac{1}{\epsilon^2}\left(k\log(k)\log^2(n)\left(\frac{\fat_\frac{\epsilon}{4}(\H)}{\epsilon^2}\log^2\left(\frac{\fat_\frac{\epsilon}{4}(\H)}{\epsilon}\right)\right)\log^2\frac{1}{\epsilon}+\log\frac{1}{\delta}\right)
\right)\right).
\end{align*}
Denote $d = \fat_\frac{\epsilon}{4}(\H)$, $A=\frac{1}{\epsilon^2}\log\frac{1}{\delta}$, and $B = k\log(k)\frac{d}{\epsilon^4}\log^2\frac{d}{\epsilon}\log^2\frac{1}{\epsilon}$.
It suffices to take $n_0=\O\paren{B\log^2 B + A}$, therefore,
\begin{align*}
    n
    &\leq
     \O 
     \left(
     \frac{1}{\epsilon^2}
     \left(
      k\log(k)\frac{d}{\epsilon^2}\log^2\frac{d}{\epsilon}\log^2\frac{1}{\epsilon}\log^2\left(k\log(k)\frac{d}{\epsilon^4}\log^2\frac{d}{\epsilon}\log^2\frac{1}{\epsilon}\right)
      +
     \log\frac{1}{\delta}
     \right)
     \right).
\end{align*}
\end{proof}
\newpage
\section*{Acknowledgments}
We deeply thank the anonymous reviewers for many insightful comments and suggestions, in particular, for providing us the proof about the $\VC$-dimension of the zero-one robust loss class (Lemma \ref{lem:vc-robust-loss}).

We thank Digvijay Boob, Praneeth Netrapalli, and Yann Guermeur for pointing out a mistakes in the conference version of this paper. We thank Shay Moran, Ron Holzman, and Steve Hanneke for resolving a mistake (see Lemma \ref{lem:disambig}) and providing 
\cref{ex:2}.
We are also gratful to Uri Stemmer for fruitful discussions. IA is thankful to Doron Cohen and Netanel Rabinowitz for many helpful discussions.

The work of AK and YM was supported in part by grants from the Israel Science Foundation (ISF). The work of IA was supported in part by Kreitman School and the Vatat Scholarship from the Israeli Council for Higher Education.

\bibliography{refs}

\appendix

\section{Additional Proofs}
\begin{proof}[of Lemma \ref{lem:vc-robust-loss}]
Take an arbitrary sample $S=\{(x_1,y_1), ...., (x_n,y_n)\}$. Construct the set that contains all possible corrupted examples on inputs from $S$, 
$S_{\rho}=\bigcup_{i\in [n]}\set{z:z\in \rho(x_i)}$,
the size of $S_{\rho}$ is at most $nk$. 
Denote by $L^\rho_{\H}(S)$ the set of all possible behaviors on $S$ using functions in $L^\rho_{\H}$, and by $\H(S_{\rho})$, the set of all possible behaviors on $S_{\rho}$ using functions in $\H$.
Namely, $L^\rho_{\H}(S)=\set{(\ell(x_1,y_1),\dots,\ell(x_n,y_n)):\ell\in L^\rho_{\H}}$ and $\H(S_{\rho})=\set{(h(z_1),\dots,h(z_m)):h\in \H}$. Observe that each pattern in the set $L^\rho_{\H}(S)$ will map to at least one pattern in $\H(S_{\rho})$, implying that the size of $L^\rho_{\H}(S)$ is at most the size of $\H(S_{\rho})$. Using Sauer’s lemma, the size of $\H(S_{\rho})$ is at most $(nk)^d$, solving for $n$ such that $(nk)^d < 2^n$ yields the stated bound.
\end{proof}

\label{app:reg-hyperplane-bound-1}
\begin{proof}[Proof of Corollary~\ref{regression-bound-hyperplanes-1}]
We 
seek
an upper bounds on the following term in the case of homogeneous hyperplanes 
with norm
bounded by $1$.
\begin{align*}
 m_n(\H)
&=
\inf_{\beta\geq 0}\set{4\beta +\O\left( \sqrt{\frac{\log^4(n)}{n}}\int_{\beta}^{1} \sqrt{\fat_{c\gamma}(\H) \log\left(\frac{1}{\gamma}\right)}d\gamma \right)},
\\
&\leq
\inf_{\beta\geq 0}\set{4\beta + \O\left(\sqrt{\frac{\log^4(n)}{n}}\int_{\beta}^{1} \frac{1}{\gamma}\sqrt{\log\left(\frac{2}{\gamma}\right)}d\gamma\right)},
\end{align*}
where the inequality stems from
the bound
$\fat_{\delta}(H)\leq \frac{1}{\delta^2}$ \citep{299098}.

Compute
\beq
  \int_\beta^1
  \frac1{t}
  \sqrt{
    \log\frac{2}{t}
  }
dt
=
\frac23\paren{ (\log 2/\beta)^{3/2}-(\log2)^{3/2}},
\eeq
choosing $\beta=1/\sqrt{n}$ yields
\beq
 m_n(\H) 
&\le&
\O\paren{\sqrt{\frac{1}{n}}\log^{\frac{7}{2}}\left(n\right)}.
\eeq
The function $x \log^{{3}/{2}+\alpha} ({n}/{x})$ is monotonic increasing for $x\in (0,n/e^{3/2+\alpha}]$. Then, for sufficiently large $n$, 
$\left(\log^{7/2}(n)e^{3/2+\alpha}\right)^{2/3}\leq n$  
we have
\beq
m_n(\H)\cdot \log^{\frac{3}{2}+\alpha}\left(\frac{n}{m_n(\H)}\right)
\leq
\O\paren{\sqrt{\frac{1}{n}}\log^{\frac{7}{2}}\left(n\right)\log^{\frac{3}{2}+\alpha}\left(\frac{n}{\log^{\frac{7}{2}}\left(n\right)}\right)}
.
\eeq
\end{proof}

\begin{proof}[of Lemma \ref{lem:disambig-rule}]
  For any finite sequence $
  (x_1,y_1),\ldots,(x_k,y_k)
  $
  with $x_i\in X$, $y_i\in\set{0,1}$,
  and $x_1<\ldots<x_k$, denote by
  $\evalat{F^\star}{(x_1,y_1),\ldots,(x_k,y_k)}
  $
  the subfamily of those members of $F^\star$
  that label the point $x_i$ with $y_i$, for all $i$.
  For such a constrained subfamily, we define its {\em weight}:
  \beq
  w(\evalat{F^\star}{(x_1,y_1),\ldots,(x_k,y_k)})
  &=&
  \sum_{S}\frac1{n(S)^{d+1}},
  \eeq
  where the summation is over all nonempty subsets $S$
  of $\N\setminus\set{1,\ldots,x_k}$
  that are shattered by this subfamily,
  and $n(S)$ denotes the largest element of $S$.
  The definition applies verbatim to the special case
  where $k=0$, i.e.,
  $\evalat{F^\star}{\emptyset}=F^\star$.
Clearly, if $c$ is a prefix of $c'$,
then
$
w(\evalat{F^\star}{c})
\ge
w(\evalat{F^\star}{c'})
$,
and hence the maximum weight is achieved by
$\evalat{F^\star}{\emptyset}=F^\star$.
The latter is upper-bounded by
\beqn
\label{eq:w(F)}
w(F^\star) \le \sum_{n\in\N}\frac{n^{d-1}}{n^{d+1}}
=
\sum_{n\in\N}\frac1{n^2}
=\frac{\pi^2}{6},
\eeqn
where the numerator $n^{d-1}$
accounts for the number of 
of 
subsets of $[n]$ of size at most $d$ which have $n$ as their largest element.

Any constrained subfamily
  $\evalat{F^\star}{(x_1,y_1),\ldots,(x_k,y_k)}$
  induces the ``majority'' classifier
  $M[\evalat{F^\star}{(x_1,y_1),\ldots,(x_k,y_k)}]
  :X\to\set{0,1}
  $ as follows:
  \beqn
  \label{eq:halving}
M[\evalat{F^\star}{(x_1,y_1),\ldots,(x_k,y_k)}]
(x)
  &=&
  \pred{
  w(\evalat{F^\star}{(x_1,y_1),\ldots,(x_k,y_k),(x,1)})
  >
  w(\evalat{F^\star}{(x_1,y_1),\ldots,(x_k,y_k),(x,0)})
  }
  \eeqn
(ties may be broken arbitrarily, and the rule
above favors $0$ in such cases).
We observe that
\beq
  w(\evalat{F^\star}{(x_1,y_1),\ldots,(x_k,y_k)})
  &\ge&
  w(\evalat{F^\star}{(x_1,y_1),\ldots,(x_k,y_k),(x,1)})
+
w(\evalat{F^\star}{(x_1,y_1),\ldots,(x_k,y_k),(x,0)})
,
\eeq
with equality occurring iff no $f^\star\in F^\star$
verifies $f^\star(x)=\star$.

We now describe the disambiguation procedure.
We proceed one ``row'' $f^\star\in F^\star$
at a time.
For a given $f^\star\in F^\star$, initialize
the ``constraint'' sequence $c$ to be empty
(i.e., to be of length $k=0$).
Predict the label at $x=1$
via 
$y=M[\evalat{F^\star}{c}](x)$.
The prediction is said to be a
{\em mistake} if $f^\star(x)\neq\star$ and $y\neq f(x)$.
In case of a mistake, append $(x,f^\star(x))$
to the end of the constraint sequence $c$
and leave $c$ unchanged otherwise.
Repeat the procedure for $x=2$:
predict $y=M[\evalat{F^\star}{c}](x)$
and append $(x,f^\star(x))$ to $c$ in case of a mistake.
Repeating these steps for $x=1,2,\ldots,m$
produces a disambiguation $\bar f$ of $f^\star$.
To disambiguate the next ``row'' of $F^\star$,
re-initialize $c:=\emptyset$ and repeat
the procedure above for $x=1,2,\ldots,m$.

Having described the construction of $\bar F$,
it remains to analyze the number of behaviors
that it can possibly attain on
a prefix of length $m$ --- that is, to bound
$|\bar F(X_m)|$. The first key observation is that
if $c$ is the constraint before a mistake
and $c'$ immediately after, then
(\ref{eq:halving}) implies that
$w(\evalat{F^\star}{c})\ge\frac12w(\evalat{F^\star}{c'})$
(i.e., the weight of the constrained family
is reduced by a half or more).
This is because a mistake is caused
by the {\em majority} being wrong,
and the updated constraint effectively removes
those members of $F^\star$ that contributed to the mistake.
The second key observation is 
that
if some
$x\le m$ witnesses the last\footnote{The case where
no mistakes are made is trivial.}
mistake when disambiguating a given $f^\star$,
the weight prior to updating the constraint
on this mistake is at least
$1/m^{d+1}$ --- because in this case, $\set{x}$ must
be a shattered set.

Together with (\ref{eq:w(F)}), these two estimates
on the weight immediately prior to the last update
imply that the number of updates $u$ satisfies
\beq
\frac1{m^{d+1}}2^{u-1} \le w(F^\star) \le \frac{\pi^2}6,
\eeq
which implies that $u=\O(d\log m)$.
To translate this into an estimate on 
$|\bar F(X_m)|$,
observe that any
$\bar f\in\bar F$ is uniquely defined by
the indices on which a mistake was made
during its disambiguation procedure.
It follows that
$|\bar F(X_m)|
\le
\O({m\choose u})
\le
m^{\O(d\log m)}
$.
\end{proof}

\paragraph{Additional Generalization Bound for Binary Classification.} 
\label{app:gen-bound-binary-additional}
We derive the result in \cref{binary_valued_worse_bound}.
Denote the sample size $|S|= n$ and $\VC(\H)=d$. Using Theorem \ref{thm:regression-bound-2} for binary valued function classes we upper bound the empirical Rademacher complexity on the sample $R_n(\max((\conv(\F^{(j)}_{\H}))_{j\in[k]})|\mathbf{x}\times \mathbf{y})$
by
\beq 
\inf_{\alpha \geq 0}
\set{4\alpha + \O\Bigg(\sqrt{\frac{k\log(k)\log^4(n)}{n}}\int_{\alpha}^{1} \sqrt{\log\left(\frac{2}{\gamma}\right)
\left(\frac{\fat_\frac{c\gamma}{4}(\H)}{\gamma^2}\log^2\left(\frac{\fat_\frac{c\gamma}{4}(\H)}{\gamma}\right)\right)}d\gamma\Bigg)}.
\eeq
For a binary valued class this is upper bounded by 
\beq 
\inf_{\alpha \geq 0}
\set{4\alpha + \O\left(\sqrt{\frac{dk\log(k)\log^4(n)}{n}}\int_{\alpha}^{1} \sqrt{\log\left(\frac{2}{\gamma}\right)\left( \frac{1}{\gamma^2}\log^2\left(\frac{d}{\gamma}\right)\right)}d\gamma\right)} \\
=
\inf_{\alpha \geq 0}
\set{4\alpha + \O\left(\sqrt{\frac{dk\log(k)\log^4(n)}{n}}\int_{\alpha}^{1} \frac{1}{\gamma}\log\left(\frac{d}{\gamma}\right)\sqrt{\log\left(\frac{2}{\gamma}\right) }d\gamma\right) }.
\eeq
Computing
\beq
\int_{\alpha}^{1} \frac{1}{\gamma}\log\left(\frac{d}{\gamma}\right)\sqrt{\log\left(\frac{2}{\gamma}\right)}d\gamma 
&=& 
\log(d)\int_{\alpha}^{1} \frac{1}{\gamma}\sqrt{\log\left(\frac{2}{\gamma}\right)}d\gamma + \int_{\alpha}^{1} \frac{1}{\gamma}\log\left(\frac{1}{\gamma}\right)\sqrt{\log\left(\frac{2}{\gamma}\right)}d\gamma \\
&\leq&
\log(d)\int_{\alpha}^{1} \frac{1}{\gamma}\sqrt{\log\left(\frac{2}{\gamma}\right)}d\gamma + \int_{\alpha}^{1} \frac{1}{\gamma}\log^{\frac{3}{2}}\left(\frac{2}{\gamma}\right)d\gamma \\
&=&
\frac{2}{3}\log(d)\left( \log^{\frac{3}{2}}\left(\frac{2}{\alpha}\right)-\log^{\frac{3}{2}}(2)\right) - \frac{2}{5} \left(\log ^{\frac{5}{2}}(2)-\log ^{\frac{5}{2}}\left(\frac{2}{\alpha }\right)\right) \\
&\leq&
\log(d)\log^{\frac{3}{2}}\left(\frac{2}{\alpha}\right) + \log ^{\frac{5}{2}}\left(\frac{2}{\alpha }\right)
\eeq
and
choosing $\alpha=\frac{1}{\sqrt{n}}$ yields
\beq
\log(d)\log^{\frac{3}{2}}(2\sqrt{n}) + \log ^{\frac{5}{2}}\left(2\sqrt{n}\right) \leq \O\left(\log(d)\log ^{\frac{5}{2}}(n)\right)
\eeq

and
\beq
R_n(\max((\conv(\F^{(j)}_{\H}))_{j\in[k]})|\mathbf{x}\times \mathbf{y})
\leq
\O\left(\sqrt{\frac{d\log^2(d)k\log(k)\log^9(n)}{n}}\right).
\eeq

\end{document}